\documentclass[11pt,letterpaper]{article}
\usepackage{fullpage}

\usepackage{cite}
\usepackage{hyperref}

\usepackage{colortbl}

\usepackage{paralist}
\usepackage{pdfpages}
\usepackage{enumitem}
\usepackage{booktabs}
\usepackage[override]{cmtt}
\usepackage{color,xcolor}
\usepackage{graphicx,color,eso-pic}
\usepackage{amsmath,amssymb,stmaryrd}

\usepackage[
	lambda,
	advantage,
	operators,
	sets,
	adversary,
	landau,
	probability,
	notions,
	logic,
	ff,
	mm,
	primitives,
	events,
	complexity,
	asymptotics,
	keys]
{cryptocode}

\definecolor{darkgreen}{rgb}{0,0.5,0}
\definecolor{lightblue}{RGB}{0,176,240}
\definecolor{darkblue}{RGB}{0,112,192}
\definecolor{lightpurple}{RGB}{124, 66, 168}
\definecolor{grey}{RGB}{139, 137, 137}
\definecolor{maroon}{RGB}{178, 34, 34}
\definecolor{green}{RGB}{34, 139, 34}
\definecolor{types}{RGB}{72, 61, 139}
\definecolor{gold}{rgb}{0.8, 0.33, 0.0}

\definecolor{darkgray}{gray}{0.3}

\newcommand{\skiptext}[1]{}
\newcommand{\pr}[1]{\text{Pr} [#1]}
\renewcommand{\Pr}{\text{Pr}}

\usepackage{boxedminipage}
\usepackage{cite, bm}
\usepackage{url}
\usepackage{graphicx}
\usepackage[bf]{caption}
\usepackage{subcaption}
\usepackage{color}
\usepackage{xspace}
\usepackage{multirow}
\usepackage{amsmath,amsthm,amstext,amssymb,amsfonts,latexsym}
\usepackage{wrapfig}

\definecolor{darkred}{rgb}{0.5, 0, 0}
\definecolor{darkgreen}{rgb}{0, 0.5, 0}
\definecolor{darkblue}{rgb}{0,0,0.5}

\newcommand\markx[2]{}

\newcommand{\bsigma}{\bm{\sigma}}

\newcommand{\G}{\mathcal{G}}

\newcommand{\ignore}[1]{}

\newcounter{task}

{
\theoremstyle{definition}
\newtheorem{theorem}{Theorem}

\newtheorem{corollary}{Corollary}
\newtheorem{definition}{Definition}
\newtheorem{lemma}{Lemma}
\newtheorem{proposition}{Proposition}

\newtheorem{claim}{Claim}
}

\newcounter{cnt:challenge}

\newcommand{\commentout}[1]{}

\newcommand{\C}{\mathcal{C}}
\newcommand{\D}{\mathcal{D}}

\let\P\oldP
\newcommand{\P}{\mathcal{P}}
\newcommand{\advO}{\mathcal{O}}
\let\S\oldS
\newcommand{\S}{\mathcal{S}}

\newcommand{\E}{{\ensuremath{\mathbf{E}}}\xspace}

\newcommand{\Out}{\Omega}
\newcommand{\Dec}{\Omega_{\text{DM}}}
\newcommand{\Class}{\Psi}

\newcommand{\T}{\mathcal{T}}
\renewcommand{\O}{O}

\newcommand{\Risky}{\mathsf{Risky}}
\newcommand{\Safe}{\mathsf{Safe}}

\newcommand{\M}{\mathcal{M}}

\newcommand{\define}[1]{\textbf{#1}}
\newcommand{\Bit}{\lbrace 0,1 \rbrace}
\newcommand{\Div}{D_\infty}
\newcommand{\md}{\mu}
\newcommand{\prc}[2]{\text{Pr}_{\sigma \leftarrow \D_{#2}} \left[ #1 \right]}

\newcommand{\braces}[1]{\lbrace #1 \rbrace}

\newcommand{\fairtreat}{\text{fair treatment}}
\newcommand{\ebr}{\text{equal base rates}}
\newcommand{\perfpred}{\text{perfect prediction}}

\begin{document}

\title{Paradoxes in Fair Computer-Aided Decision Making}

\author{
Andrew Morgan
\\ Cornell University
\\
\texttt{asmorgan@cs.cornell.edu}
\and
Rafael Pass\thanks{Supported in part by NSF Award CNS-1561209, NSF Award CNS-1217821, 
AFOSR Award FA9550-15-1-0262,  
a Microsoft Faculty Fellowship, and a Google Faculty Research Award.}\\Cornell Tech\\\texttt{rafael@cs.cornell.edu}
}

\maketitle

\begin{abstract}
Computer-aided decision making---where 
%raf14
%some classifier (e.g., an algorithm trained
a human decision-maker is aided by a computational classifier in making
a decision---is becoming increasingly
prevalent. For instance, 
judges in at least nine states make use of algorithmic tools meant to
determine ``recidivism risk scores'' for criminal defendants in sentencing, parole, or bail decisions.
A subject of much recent debate is whether such algorithmic tools are
``fair'' in the sense that they do not discriminate against
certain groups (e.g., races) of people. 
%raf7: new

Our main result shows that for ``non-trivial'' computer-aided decision
making, either the
classifier must be discriminatory, or a rational
decision-maker using the output of the classifier is forced to be
discriminatory. We further provide a complete characterization of
situations where fair computer-aided decision making is possible.

\end{abstract}

%raf6
\newpage
\section{Introduction}
As more and more data is becoming easily available, and with vast increases in
the power of machine learning, there are an increasing number of
situations where algorithms---\emph{classifiers}---are used to help decision makers in
challenging situations. Examples range from algorithms assisting
drivers in cars, to algorithmic methods for determining credit scores, to
algorithms helping judges to make sentencing and pretrial decisions in criminal justice. 
While such \emph{computer-aided decision making} has presented
unparalleled levels of accuracy
and is becoming increasingly ubiquitous,
one of the primary concerns with its widespread adoption is the possibility for such algorithmic methods to lead to structural biases and discriminatory practices \cite{united}. 
A malicious algorithm designer, for instance, might explicitly
encode discriminatory rules into a classifier. Perhaps even more problematically, a machine learning
method may overfit the data and infer a bias, may inherit a bias
from poorly collected data, or may simply be designed to optimize some
loss function that leads to discriminatory outcomes.

A well-known instance where this concern has come to light is the debate surrounding the COMPAS (Correctional Offender Management Profiling for Alternative Sanctions) tool for recidivism analysis, a classification algorithm that is becoming increasingly widely used in the criminal justice system. Given a series of answers to questions concerning criminal defendants' backgrounds and characteristics, this tool outputs scores from 1 (low risk) to 10 (high risk) estimating how likely they are to recidivate (commit a future crime) or to recidivate violently. According to a recent study by ProPublica \cite{compas}, COMPAS and similar risk assessment algorithms are becoming increasingly widely-used throughout the United States; their results are already being shown to judges in nine states during criminal sentencing, and are used in courts nationwide for pretrial decisions such as assigning bail.
%\note{can we give some number about how
%much better COMPAS-aided judges are??}
%
The ProPublica study, however, found an alarming trend in a set of data collected \cite{compas2} concerning individuals' COMPAS results and their actual rates of recidivism over the next two years; in particular, it was found that the scores output by COMPAS lead to a \emph{disparate treatment} of
minorities. For instance, in the data collected, African-American defendants \emph{who did not recidivate} were found to be almost twice as likely as white defendants (44.85\%, compared to 23.45\%) to have been assigned a high risk score (5-10).
%These findings were later confirmed by Chouldechova in \cite{chouldechova}, who noted that the COMPAS classifier was designed according to the alternate and largely incompatible heuristic of ``predictive parity"---that is, white and African-American defendants having the same score are roughly equally likely to actually recidivate. Hence, a question naturally arising from this and other similar debates is that of which of the many different definitions and notions of algorithmic fairness are natural, desirable, and achievable.

Fairness, or non-discrimination,  in classification has been studied
and debated extensively in the recent past (see \cite{barocas} for an extremely thorough overview);
research concerning definitions of fairness in classification dates back to the works of
Pearl \cite{pearl}  and Dwork et al. \cite{dwork}, with more recent
definitions tailored to deal with the above-mentioned problems
appearing in \cite{compas, hardt, chouldechova, kleinberg}.
%\paragraph{Fair Treatment in Classification}
To make this setting more concrete, consider some
distribution $\D$ from which \define{individuals} $\sigma$ are sampled, and
consider some \define{classifier} $\C(\cdot)$ that given some observable
features $\O(\sigma)$ produces some \define{outcome}, which
%raf2: \Out_C is not defined...do we really need it
% $\C(\O(\sigma)) \in \Out_\C$ which
later will be used by a decision-maker (DM). The DM is ultimately
only interested in the actual \define{class} $f(\sigma) \in \Class$ of the
individual $\sigma$, and their goal is to output
some decision $x \in \Dec$ correlated with this actual class.
For instance, in the setting of the COMPAS data collected in \cite{compas, compas2}, $\D$ is the distribution over defendants $\sigma$, the
class $f(\sigma)$ is 
%raf6
a bit indicating
whether the 
%individual 
defendant $\sigma$
actually commits a crime
in the next two years, and the job of the classifier is to output a
risk score, which will then be seen and acted upon by a judge. 
Note that we may without loss of generality assume that
the class of the individual is fully determined by
$\sigma$---situations where the class is probabilistically decided (e.g., at the time of classification, it has yet to be determined whether an individual will or won't recidivate) can
be captured by simply including these future coin-tosses needed to
determine it into $\sigma$, and simply making sure they are not part of
the observable features $\O(\sigma)$.

Additionally, an individual $\sigma$ is part of some \define{group}
$g(\sigma) \in \G$---for instance, in the COMPAS setting, the group is the
race of the individual. We will refer to the tuple $\P = (\D,f,g,\O)$ as a
\define{classification context}. Given such a classification context
$\P$, we let $\Class_{\P}$ denote the range of $f$, and $\G_{\P}$
denote the range of $g$. Whenever the classification context $\P$ is
clear from context, we drop the subscript; 
%raf2
additionally, whenever the distribution $\D,g$ are clear from context, we
use $\bsigma$ to denote a random variable that is distributed
according to $\D$, and $\bsigma_X$ to denote the random variable
distributed according to $\D$ conditioned on $g(\sigma) = X$.

In this work, we will explore a tension between \emph{fairness for the classifier} and \emph{fairness for the
DM}. Roughly speaking, our main result shows that except in
``trivial'' classification contexts, either the classifier needs to be discriminatory, or a rational
decision-maker using the output of the classifier is forced to be
discriminatory.
%raf6
Let us turn to describing these two different perspectives on fairness.

\paragraph{Fairness for the Classifier: Fair Treatment}
The notion of \emph{statistical
  parity} \cite{dwork} (which is essentially identical to the notion
of causal effect \cite{pearl}) captures non-discrimination between groups by simply requiring that the
output of the classifier be independent (or almost independent) of the
group of the individual; 
%raf6
%That is, given a context $\P = (\D, f, g, O)$, for any two groups $X,Y
%\in \G_\P$, the probability distributions
%andrew1: might as well use the context here
%raf6: sorry, i prefer not, too clunky and we also don't quantify over \C.
that is, for any two groups $X$ and $Y$, the distributions 
%\begin{itemize}
%\item 
%raf2: make notation even simpler
%$\braces{ \C(\O(\bsigma)) | g(\bsigma) = X}$, $\braces{
%  \C(\O(\bsigma)) | g(\bsigma) = Y}$
$\braces{ \C(\O(\bsigma_X))}$ and $\braces{ \C(\O(\bsigma_Y))}$
%\end{itemize}
are $\epsilon$-close in statistical distance.
%, where $\bsigma$ is a
%random variable distributed according to $\D$. (In the sequel,
This is a very strong notion of fairness, and in the above-mentioned
context it may not make sense. In particular, if the \emph{base rates} (i.e.
the probabilities that individuals from different groups are part of a
certain class) are different, we should perhaps not expect the output
distribution of the classifier to be the same across groups. Indeed,
as the ProPublica article points out, in the COMPAS example, the
overall recidivism probability among African-American defendants was
56\%, whereas it was 42\% among white defendants. Thus, in such
situations, one would reasonably expect a classifier to 
%raf6
%on average output a higher
\emph{on average} output a higher
risk score for African-American defendants, which would violate statistical parity. Indeed, the issue raised
by ProPublica authors was that, even after taking this base difference into
account (more precisely, even after conditioning on individuals that
did not recidivate), there was a significant difference in how the classifier
treated the two races. 

The notion of \emph{equalized odds} due to Hardt et al. \cite{hardt} formalizes the desiderata
articulated by 
%raf6
the authors of the ProPublica study 
(for the case of recidivism) in a general
%raf6
%setting by requiring the output of the classifier be independent
setting by requiring the output of the classifier to be independent
of the group of the individuals, \emph{after
conditioning on the class of the individuals}.\footnote{Very similar
notions of fairness appear also in the works of
\cite{chouldechova,kleinberg} using different names.} We here consider an
approximate version of this notion---which we refer to as
\define{$\epsilon$-fair treatment}---which requires that, 
%raf6: reverting
%given context $\P = (\D, f, g, O)$, for any two
%groups $X,Y \in \G_\P$, and any class $c \in \Psi_\P$, the distributions
%andrew1: using the context here too
%raf6: why? these are all informal definitions..
% treatment}---which requires that 
for any two
groups $X$ and $Y$ and any class $c$, the distributions
\begin{itemize}
%raf2*: not sure which is cleaner.
%\item $\braces{ \C(\O(\bsigma)) | f(\bsigma) = c , g(\bsigma) = X}$
%\item $\braces{ \C(\O(\bsigma)) | f(\bsigma) = c , g(\bsigma) = Y}$
\item $\braces{ \C(\O(\bsigma_X)) \mid f(\bsigma_X) = c}$
\item $\braces{ \C(\O(\bsigma_Y)) \mid f(\bsigma_Y) = c}$
\end{itemize}
are $\epsilon$-close with respect to some appropriate distance metric to be defined
shortly.
That is, in the COMPAS example, if we restrict to individuals that
actually do not recidivate (respectively, those that do), the output of the
classifier ought to be essentially independent of the group of the individual
(just as intuitively desired by the authors of the ProPublica study,
and as explictly put forward in \cite{hardt}).

%raf2: i kept max-div here as it makes the informal def easier, but we
%don't need to use it in the formal defs (although we actually could;
%we can just have a one-line def of max-div using \mu)
We will use the notion of \emph{max-divergence} to
determine the ``distance" between distributions; this notion,
often found in areas such as differential privacy (see \cite{dwork2}),
represents this distance as (the logarithm of) the \emph{maximum multiplicative
gap} between the probabilities of some element in the respective distributions.
We argue that using such a multiplicative distance is important to
ensure fairness between groups that may be under-represented in the data (see
Appendix \ref{sec:varcomp}).
Furthermore, as we note (see Theorem \ref{clm:bracket} in Appendix \ref{threshold.ssec}), such a notion is closed under
``post-processing'': if a classifier $\C$ satisfies
$\epsilon$-fair treatment with respect to a context $\P = (\D, f, g,
O)$, then for any (possibly probabilistic) function 
%raf2: \Out_\C is not defined.. so cannot quantify like this
%$\M : \Out_\C \rightarrow \Out_{\C'}$, $\C'(\cdot) = \M(\C(\cdot))$ will
$\M$, $\C'(\cdot) = \M(\C(\cdot))$ will
also satisfy $\epsilon$-fair treatment with respect to $\P$. Closure under post-processing
%andrew1: using the context
%$\epsilon$-fair treatment, then for any (possibly probabilistic) function $\M : \Out_\C \rightarrow \Out_{\C'}$, $\C'(\cdot) = \M(\C(\cdot))$ will
%also satisfy $\epsilon$-fair treatment. Closure under post-processing
is important as we ultimately want the decision-maker to act on the
output of the classifier, and we would like the decision-maker's output to be fair whenever they act only on the classifier's output.\footnote{We note that an earlier approximate definition was proposed by
Kleinberg et al. \cite{kleinberg}, which simply required that the
expectations of the distributions are close; while this is equivalent to our definition for the case of
binary outcomes, it is weaker for non-binary
outcomes (as in the case of the COMPAS classifier), and, as we note (also in Appendix \ref{sec:varcomp}),
this notion is not closed under post-processing.}

As shown in the ProPublica study, the COMPAS classifier
does not satisfy $\epsilon$-fair treatment even for somewhat large
$\epsilon$. However, several recent works have presented methods to ``sanitize'' unfair classifiers
into ones satisfying $\epsilon$-fair treatment with only a relatively
small loss in accuracy \cite{hardt, zafar, morganpass2}.

\paragraph{Fairness for the Decision-Maker: Rational Fairness.}
%raf1: i liked the old flow better here
%raf6: andrew, first, mark you changes so i can spot them. second,
%there is a jump where that is non-trivial which is why the sentence
%was there. unless FT was possible, there would be no point to
%discussing the rest.
%raf6: rewrote a bit to appease you 
So, classifiers 
satisfying $\epsilon$-fair treatment with accuracy closely matching
the optimal ``unfair'' classifiers are possible (in fact,
classifiers 
such as COMPAS can be sanitized to
satisfy $\epsilon$-fair treatment,
%ensure that different groups of individuals are treated fairly by the classifier
without losing too much in accuracy).
%raf6:
Additionally, as we
%We 
have noted,
%that
the notion of fair treatment is closed under post-processing, so any
mechanism that is applied to the output of the classifier will preserve
fair treatment. Thus, intuitively, we would hope that 
the entire ``computer-aided decision making process'', where the
decision-maker makes use of the classifier's output to make a decision, results in
a fair outcome as long as the classifier satisfies fair treatment. Indeed, if the
decision-maker simply observes the outcome of the classifier and bases
their decision entirely on this outcome, this will be the case 
%raf1: added this based on your suggestion
(by
closure under post-processing).
\iffalse 
So, since
the notion of fair treatment is closed under post-processing, any
mechanism that is applied to the output of the classifier will preserve
fair treatment. Intuitively, we would hope that 
the entire computer-aided decision making process, where the
decision-maker makes use of the classifier's output to make a decision, results in
a fair outcome as long as the classifier satisfied fair treatment. And indeed, by post-processing, if the
decision-maker simply observes the outcome of the classifier and bases
their decision entirely on this outcome, this will be the case.
\fi

But the decision-maker is not a machine; rather we ought to think of
the DM as a \emph{rational
agent}, whose goal is to make decisions that maximize some internal utility function.
(For instance, in the context of COMPAS, the DM might be a judge that
wants to make sure that defendants that are likely to recidivate are
sent to jail, and those who do not are released). 
As far as we are aware, such a
\emph{computer-aided ``rational'' decision-making} scenario has not
previously been studied.

More precisely, the
DM is actually participating in a Bayesian game \cite{harsanyi}, where
individuals $\sigma$ are sampled from $\D$, the DM gets to see the
group, $g(\sigma)$, of the individual and
the outcome, $c = \C(\sigma)$, of the classifier (e.g., the individual's race and risk score), and then needs to select some action $x
\in \Dec$ (e.g., what sentence to render), and finally receives some utility
$u(f(\sigma), x)$ that
is only a function of the actual class $f(\sigma)$ (e.g., whether the
individual would have recidivated) and their decision $x$. 
%raf2 :skipping this for now :-)
%For
%simplicity (and as it only makes our results stronger) we restrict to
%the case where the action space of the game is simply $\Dec = \Bit$
%(and our results apply also to larger finite outcome spaces $\Dec$).
%andrew1: making sure that the reader understands that it still works,
%we need to be clear about this. we need to say it again in the
%theorem foreword.
%raf2*: in fact, now when we have a tight characterization, we may
%want to consider the action space explicitly...since in the backward
%direction, we may want to show that it works for any finite game.
%the case where the action space of the game is simply $\Dec = \Bit$.

Given a classification context $\P$, a classifier
$\C$, 
%raf2: added
action space $\Dec$ 
and a utility function $u$, let $\Gamma^{\P,
  \C, \Dec, u}$ denote the Bayesian game induced by the above process.
We argue that in a computer-aided decision-making scenario, 
a natural fairness desideratum for a classifier $\C$ for
a context $\P$ is that
it should ``enable fair rational decision-making''.
%andrew1: let's be more specific
%raf2*: same issue with \Out_C, not defined...you can recplace with;
%can you fix it. (please mark how you fix)
%something like \Out_\P which would be define.
%raf2*
%More precisely, we say that a strategy $s: \G_\P \times \Out_\C
%\rightarrow \Dec$ for the DM (which chooses a decision based on the
%group and output) is
More precisely, we say that a strategy $s: \G_\P \times \{0,1\}^*
\rightarrow \Dec$ for the DM (which chooses a decision based on the
group of the individual and the output of the classifier) is
\emph{fair} if the DM ignores the group $g$ of the individual
and only bases its decision on the output of the classifier---that is,
there exists some $s' : \{0,1\}^* \rightarrow \Dec$ such that $s(g,o) = s'(o)$. 
We next say that $\C$ \define{enables $\epsilon$-approximately fair decision making} (or simply
satisfies \define{$\epsilon$-rational fairness}) with respect to the context
$\P = (\D, f, g, O)$ if, for every 
%raf2: added
finite action space $\Dec$ and every
utility function $u:\Class \times \Dec \rightarrow [0,1]$
(i.e., depending on the individual's class and the 
%raf2:
%decision made), there exists an
action selected by the DM), there exists an
$\epsilon$-approximate Nash equilibrium $s$ in the induced game 
%raf2
%$\Gamma^{\P, \C,  u}$ where $s$ is fair.
$\Gamma^{\P, \C, \Dec, u}$ where $s$ is fair.

%raf2:
%Note that if there does not exist some fair $\epsilon$-Nash
%equilibrium, a DM can gain more than
Note that if there exists $\Dec,u$ for which there there does not
exist some fair $\epsilon$-Nash equilibrium in the induced game, then
there exist situations in which a DM can gain more than
$\epsilon$ in utility by discriminating between groups, and thus in
such situations a rational DM (that cares about ``significant" $> \epsilon$ changes in
utility) would be \emph{forced} to do so. 

%raf6
%\subsection{The Impossibility Result}
\subsection{Our Main Theorem}
Our main 
%raf6
%impossibility 
result shows that the above-mentioned notions
of fairness---which both seem intuitively desirable---are largely
incompatible, except in ``trivial'' cases. In fact, we provide a tight
characterization of classification contexts that admit classifiers
satisfying $\epsilon$-fair treatment and $\epsilon$-rational fairness.

\paragraph{The case of binary classes (warm-up).}
As a warm-up, let us start by explaining our characterization for the case of binary
classes. 
%raf12
(Considering the binary case will also enable better comparing our
result to earlier results in the literature that exclusively focused
on the binary case.)
We  refer to a a
classification context $\P = (\D,f,g,\O)$ as binary if $\Class_{\P} =
\Bit$.

We say that a  binary classification context $\P = (\D, f, g, O)$ is \define{$\epsilon$-trivial} if \emph{either} (a) for every class $c \in \Bit$, the
``base rates'' of $c$ are $\epsilon$-close with respect to any pair of groups, \emph{or} (b) the observable features
enable \emph{perfectly distinguishing} between the two
classes. Formally, \emph{either} of the following conditions hold:
\begin{itemize}
\item 
%raf6
\emph{(``almost equal base rates''):}
for any two groups
  $X,Y$ in $\G_{\P}$, and any class $c \in \Class_{\P}$, the multiplicative distance between 
%raf2
%  $\Pr [ f(\bsigma) = c | g(\bsigma) = X}$ and $\Pr [ f(\bsigma) = c |
%g(\bsigma) = Y}$ is at most $\epsilon$, or
  $\Pr [ f(\bsigma_X) = c]$ and $\Pr [ f(\bsigma_Y) = c]$ is at most $\epsilon$;
\item 
%raf6
\emph{(``perfect distinguishability''):}
the distributions $\braces{ \O(\bsigma) \mid f(\bsigma) = 0}$ and $\braces{ \O(\bsigma) \mid f(\bsigma) = 1}$
have disjoint support.
\end{itemize}
Note that if base rates are $\epsilon$-close, there is a trivial
%raf1:
%classifier that satisfies $\epsilon$-fair treatment and
classifier that satisfies $0$-fair treatment and
$\epsilon$-rational fairness: namely, ignore the input and simply
output some canonical value. Additionally, note that if the observable
features fully determine the class of the individual, there also
%raf1:
%exists a classifier trivially satisfying \emph{exact} ($\epsilon =
%0$) fair treatment and rational
exists a classifier trivially satisfying $0$-fair treatment and $0$-rational
%raf1:
%fairness: simply distinguish and output the correct class of the
%individual.
fairness: simply output the correct class of the individual based on
the observable features (which fully determine it by assumption).
So $\epsilon$-trivial binary classification contexts admit classifiers
satisfying $\epsilon$-fair treatment and $\epsilon$-rational
fairness. Our characterization result shows that the above contexts are the only ones which
admit them.

\begin{theorem}
\emph{(Characterizing binary contexts.)} 
%raf6:
Consider a binary classification context $\P = (\D, f, g, O)$, and let
$\epsilon \leq 3/2$ be a constant. Then:
\begin{enumerate}
\item[(1)] If $\P$ is $\epsilon$-trivial, there exists a 
%
%binary 
classifier $\C$ 
%(i.e., $\Out_\P^\C = \Bit$) 
satisfying 0-fair treatment and $2\epsilon$-rational fairness with respect to $\P$.
%raf6: why do you restrict to binary classifiers?
\item[(2)] 
%If there exists a binary classifier $\C$ that, for $\epsilon \in (0, 3/2)$, satisfies $\epsilon$-fair treatment and $\epsilon / 5$-rational fairness with respect to $\P$, then $\P$ is $4\epsilon$-trivial.
If there exists a classifier $\C$ satisfying $\epsilon$-fair treatment and $\epsilon / 5$-rational fairness with respect to $\P$, then $\P$ is $4\epsilon$-trivial.
\end{enumerate}
\label{thm:informal}
\end{theorem}
As we shall discuss in more detail in Section \ref{related.sec}, a similar notion of
triviality was considered in \cite{kleinberg,chouldechova} who
obtained related characterizations for binary classifications tasks
(but considering different notions of fairness/accuracy).

\paragraph{The general case.}
To deal with the general (i.e., non-binary) case, we need to consider
a more general notion of a trivial context.
%raf12:
The definition of triviality is actually somewhat different from the
definition given for the binary case, but its not hard to see that for
this special case the definitions are equivalent.
%raf13:
%is equivalent

%a classification context where $|\Class| = k$) is
We say that a classification context $\P = (\D, f, g, O)$ is \define{$\epsilon$-trivial}
if there exists a partition of the set $\Class_{\P}$ into 
%raf12*: i think this is wrong: they need not be proper! if so, we do
%not capture the case when there is only 1 class (and thus we have
%equal base rates)
%disjoint proper 
subsets $\Class_1, \Class_2,
\ldots, \Class_m$ of classes such that 
%raf12
\emph{both} of
the following conditions hold:
\begin{itemize}
\item 
%raf6: restructured
\emph{(``base-rates conditioned on $\Class_i$ are close''):}
for any $i \in [m]$, for any two groups
  $X,Y$ in $\G_{\P}$, and any class $c \in \Class_i$, the
  multiplicative distance between 
%raf2
%  $\Pr [f(\bsigma) = c |  | f(\bsigma) \in \Class_i, g(\bsigma) = X}$ and 
 % $\Pr [f(\bsigma) = c |  | f(\bsigma) \in \Class_i, g(\bsigma) = Y}$
  $\Pr [f(\bsigma_X) = c  \mid f(\bsigma_X) \in \Class_i]$ and 
 $\Pr [f(\bsigma_Y) = c  \mid f(\bsigma_Y) \in \Class_i]$
is at most $\epsilon$;
\item 
\emph{(``perfect distinguishability between $\Class_i$ and $\Class_j$''):}
for any $i \neq j \in [m]$ the distributions $\braces{\O(\bsigma)
   \mid f(\bsigma) \in \Class_i}$ and $\braces{\O(\bsigma)
   \mid f(\bsigma) \in \Class_j}$ have disjoint support.
\end{itemize}
%raf12:
Note that in contrast to the definition given for binary context, the
general definition requires that \emph{both} of the above conditions
hold (as opposed to just one of them). Note, however, that in case we
only have 2 classes, there are only 2 possible partitions of
$\Class_{\P}$: either we have the trivial partition $\Class_1 = \{0,1\}$
in which case condition 1 is equivalent to requiring equal
base rates, and condition 2 trivially holds; or $\Class_1 = \{0\},
\Class_2 = \{1\}$, in which case condition 1 trivially holds, and
condition 2 is equivalent to prefect distinguishability between class
0 and class 1.

%raf12: removed
%Note that if $\P$ is a binary classification context, this definition is
%equivalent to our earlier definition of triviality
% that was specialized
% for binary classification contexts.

Once again, if a classification context is $\epsilon$-trivial,
there exists a simple classifier that satisfies $\epsilon$-fair
treatment and $O(\epsilon)$-rational fairness: given some observable
features $o$, determine which \emph{subgroup} $\Class_i$ the individual belongs
to (which we know can be done by the second requirement), and finally output $i$. 
Roughly speaking, this classifier satisfies $0$-fair treatment since for any $i$ and any
class $c \in \Class_i$, all individuals in $\Class_i$ receive the same
outcome (namely, $i$). Rational fairness is a bit more tricky to prove,
but roughly speaking follows from the fact that, conditioned on any
classification outcome $i$, the group $g$ of the individual
carries ``$O(\epsilon)$ information'' about the actual class of the
individual, and so, by ignoring it, the DM loses little in utility.
Our main theorem shows that $\epsilon$-triviality is also a necessary condition:

\begin{theorem}
\emph{(Full characterization.)}
%raf6Assume a context $\P = (\D, f, g, O)$ for which $|\Class_\P| =
%k$.
Consider some classification context $\P = (\D, f, g, O)$, let
$\epsilon \leq 3/2$ be a constant and let $k = |\Class_\P|$ (i.e., the
number of classes). Then:
\begin{enumerate}
\item[(1)] If $\P$ is $\epsilon$-trivial, there exists a classifier $\C$ satisfying 0-fair treatment and $2\epsilon$-rational fairness with respect to $\P$.
\item[(2)] If there exists a classifier $\C$ satisfying $\epsilon$-fair treatment and $\epsilon / 5$-rational fairness with respect to $\P$, then $\P$ is $4(k-1)\epsilon$-trivial.
\end{enumerate}
\label{thm:informal2}
\end{theorem}
%raf7: no!!!
%\setcounter{theorem}{1}
%raf6: cutting
%So to conclude, for non-trivial contexts, either we must use a
%classifier that is discriminatory or a rational
%decision-maker using the output of the classifier is forced to be
%discriminatory!

\subsection{Related Work}
\label{related.sec}
Several recent works also show obstacles to achieving fair
classifications. Notably, the elegant result by Kleinberg et al.
\cite{kleinberg} shows that (in our terminology), for non-trivial \emph{binary}
classification problems, there are no classifiers that satisfy
$\epsilon$-fair treatment (in fact, an expectation-based relaxation of
the notion we consider) as well as a notion of
\emph{$\epsilon$-group calibration}---roughly speaking,
$\epsilon$-group calibration requires that conditioned on any outcome
and group, the distributions of individuals' actual classes are
(approximately) ``calibrated" according to the outcome\footnote{For
  instance, in the COMPAS example, calibration might require that, of
  people in each group assigned a risk score of 5, approximately 50\%
  will recidivate, and among those assigned a risk score of 3,
  approximately 30\% recidivate.}. Calibration, however, is best thought of as an ``accuracy''
notion for the classifier (rather than a fairness notion), and may not
always be easy to achieve even without any concern for
fairness. 
(Additionally, the results from \cite{kleinberg} show a weaker bound
than those we present here, namely that both of the
$\epsilon$-approximate notions they consider in conjunction imply
$O(\sqrt{\epsilon})$ difference in base rates or $O(\sqrt{\epsilon})$
prediction error; we present a stronger, asymptotically tight, bound
implying either $O(\epsilon)$ difference in base rates or \emph{exact}
perfect prediction. 
%raf6
However, we note that this is largely due to the fact that the actual
definitions employed are incomparable.)

\iffalse
(Additionally, the results from \cite{kleinberg} have a
quadratic loss in terms of $\epsilon$---[ANDREW  fill in what this means]---, and thus their results strictly
speaking do not rule out the existence of non-trivial fair and
calibrated classifiers.)
\fi
%andrew1: i still don't understand what you mean by this. it is not a tight bound? it doesn't fail to rule out anything, it's just a weaker bound than ours...}

A different elegant work by Chouldechova in \cite{chouldechova}
presents a similar impossibility result. She also focuses on binary
classification contexts, and further restricts her attention to
classifiers outputting a single bit. She points out a simple identity
(a direct consequence of the definition of conditional probabilities) 
%raf6: rewrote below
which implies that, in non-trivial binary classification contexts, and
for binary classifiers (i.e., classifiers only outputting a single
bit), $0$-fair
treatment is incompatible with a notion of \emph{perfect} ``predictive
parity''---namely, that conditioned on the classifier outputting $b$, the
probability that the class is $b$ is independent of the
group. 
While her result only applies in a quite limited setting (binary
context, binary classifiers, and only rules out ``perfect'' fair
treatment combined with ``perfect'' predictive parity), we will
rely on an identity similar to hers in one step of our proof. We will also
rely on a generalized version of a notion of predictive parity
(which deals with non-binary classes, non-binary outcomes, and
non-zero error in predictivity) as an
intermediate notion within the proof of our main theorem.

As far as we know, no earlier results have considered the effect of
having a rational decision-maker act based on the output of the
classifier. Furthermore, none of the earlier impossibility results consider
non-binary classification problems.

Finally, we mention that, on a conceptual level, our characterization is related
to impossibility results for social choice, such as the
Gibbard-Satterthwaite Theorem \cite{gibbard,satter}, which demonstrate
limitations of ``fair'' and strategy-proof voting rules.

\subsection{Proof Outline}
%raf67
We here provide an outline of the proof of the main theorem. We start
by considering just binary classification contexts $\P = (\D, f, g,
O)$, and then show how to extend the proof to deal also with
non-binary contexts.
% (that is, where $\Out_\P^\C = \Class_\P = \Bit$), we can present a
% simplified version of the argument we present for the general
% case.\footnote{In the body of this paper, we prove Theorem
% \ref{thm:informal2} and note that Theorem \ref{thm:informal} is a
% direct consequence.} 
As mentioned above, for binary contexts, the ``if" direction of the
theorem (i.e.,
showing that trivial contexts admit fair classifiers) is immediate. The ``only if" direction requires showing that
the existence of a classifier $\C$ that satisfies $\epsilon$-fair
treatment as well as $\epsilon/5$-rational fairness for a context $\P$ implies that $\P$
is $O(\epsilon)$-trivial.

\paragraph{Predictive parity.}
Towards showing this, we introduce a generalized version of the notion
of ``predictive parity'' considered by Chouldechova
\cite{chouldechova} (this notion will later also be useful in proving
the ``if" direction for the non-binary classification case). Roughly speaking, we say that a classifier
satisfies $\epsilon$-predictive parity if, for any two groups $X$ and $Y$, the following distributions are
$\epsilon$-close in multiplicative distance:
\begin{itemize}
%raf6: old notation!
%\item $\braces{ \sigma \leftarrow \D | \C(\O(\sigma)) = c, g(\sigma)
%= Y : f(\sigma)}$
\item $\braces{ f(\bsigma_X) \mid \C(\O(\bsigma_X)) = c}$
\item $\braces{ f(\bsigma_X) \mid \C(\O(\bsigma_X)) = c}$
\end{itemize}
That is, the output of the classifier is ``equally predictive'' of the
actual class between groups. 

\paragraph{Relating predictive parity and rational fairness.}
Our first result (which works for all, and not just binary, contexts)
shows that rational fairness and predictive parity are intimately
connected. First of all, $\epsilon$-predictive parity implies
$O(\epsilon)$-rational fairness---intuitively, if a DM could gain by
discriminating, then there must exist some output for the classifier
for which such a gain is possible, and this contradicts predictive
parity. This forward direction turns out to be useful for proving
that all $\epsilon$-trivial contexts (even non-binary ones) admit
classifiers satisfying $\epsilon$-rational fairness and
$\epsilon$-fair treatment; that is, the ``if" direction of the theorem
(also for non-binary contexts.)

More interestingly, we show that $\epsilon/5$ rational fairness (for
$\epsilon< 3/2$), \emph{combined with} $\epsilon'$-fair treatment (for
any $\epsilon')$, implies $\epsilon$-predictive parity.
Intuitively, we show this as follows. Consider some $\C$ that does not
satisfy $\epsilon$-predictive parity, yet satisfies $\epsilon / 5$-rational
fairness and $\epsilon'$-fair treatment. This means there exists some
class $y^*$, groups $g,g'$ and some outcome $o$ such that the
prevalence of $y^*$  is significantly higher in group $g$ than in
group $g'$ conditioned on the classifier outputting $o$. 

We then construct a very natural game for the DM where every fair strategy has
low utility compared to the optimal unfair strategy, which would
contradict rational fairness. The action
space of the games consists of two actions $\{\Risky, \Safe\}$. If the DM chooses
$\Safe$ they always receive some fixed utility $u^*$. On the other hand, if they
choose $\Risky$, they receive 1 if the individual's class is $y^*$ and 0
otherwise. That is, playing $\Risky$ is good if the individual is
``good'' (i.e., in class $y^*$) and otherwise not. 

We next show, relying on the fact that $\C$ satisfies fair
treatment and the fact that the prevalence of $y^*$ is significantly
higher conditioned on the DM getting the signal $(o,g)$ than when
getting $(o,g')$, that, if we set $u^*$ (i.e, the utility of
playing $\Safe$) appropriately, the DM can always significantly gain by discriminating between
$g$ and $g'$. The intriguing aspect of this proof is that the optimal
``fair'' strategy for the DM turns out to be a \emph{mixed} strategy
(i.e., a probabilistic strategy) which mixes uniformly between the two
actions $\Risky$ and $\Safe$.

%andrew7: cut "Triviality of", makes it seem like achieving fairness is trivial, rather than impossible
\paragraph{Simultaneously achieving fair treatment and predictive parity (binary contexts).}
Given that $O(\epsilon)$-rational fairness combined with
(any finite-error) fair treatment implies $O(\epsilon)$-predictive parity, to prove the theorem, it
will suffice to show that only trivial contexts admit classifiers that
simultaneously satisfy $O(\epsilon)$-fair treatment and
$O(\epsilon)$-predictive parity.

Towards showing this, let us first focus on binary classification contexts.
We first note that, by the definition of conditional probability, for
any $X \in \G_\P$, $i, j \in \Class_\P$, and $o \in \Out_\P^\C$, the
following identity holds:
$$\frac{\pr{f(\bsigma_X) = j \mid \C(O(\bsigma_X)) = o }}{\pr{f(\bsigma_X) = i \mid \C(O(\bsigma_X)) = o }} \frac{\pr{\C(O(\bsigma_X)) = o \mid f(\bsigma_X) = i }}{\pr{\C(O(\bsigma_X)) = o \mid f(\bsigma_X) = j }}  = \frac{\pr{f(\bsigma_X) = j }}{\pr{f(\bsigma_X) = i }}$$
This identity is basically a generalization of an identity
observed by Chouldechova \cite{chouldechova} for the special case of
binary classification tasks and binary classifiers; it relates the conditional probabilities defining fair
treatment and predictive parity (the first and second terms on the
left, respectively) to the \emph{base rates} of classes between any
two groups (the terms on the right). 

The same identity holds for any $Y \in \G_{\P}$:
$$\frac{\pr{f(\bsigma_Y) = j \mid \C(O(\bsigma_Y)) = o }}{\pr{f(\bsigma_Y) = i \mid \C(O(\bsigma_Y)) = o }} \frac{\pr{\C(O(\bsigma_Y)) = o \mid f(\bsigma_Y) = i }}{\pr{\C(O(\bsigma_Y)) = o \mid f(\bsigma_Y) = j }}  = \frac{\pr{f(\bsigma_Y) = j }}{\pr{f(\bsigma_Y) = i }}$$
By applying $\epsilon$-fair treatment and $\epsilon$-predictive
%andrew7: outside of proofs we should write out things like "left-hand side", "with respect to", "without loss of generality", etc.
parity, we get that the left-hand sides on the above two identities are
$4\epsilon$-close, and as a consequence we have that 
the
ratios 
$$\frac{\pr{f(\bsigma_X) = j }}{\pr{f(\bsigma_X) = i }} \quad
\text{and}  \quad \frac{\pr{f(\bsigma_Y) = j }}{\pr{f(\bsigma_Y) = i
  }}$$ 
are $4\epsilon$-close.
(Note that, to
perform these manipulations, it is important that we rely on the
multiplicative distance notion; the reason we now get a distance of
$4\epsilon$ is that we apply fair treatment or predictive parity four
times, and each time we do this we lose a factor of $e^\epsilon$).
%raf7: added
For the case of binary classification contexts, letting $\alpha^g_b =
\pr{f(\bsigma_g) = b }$ denote the ``base rate'' of class $b$
  for group $g$, this means that the ratios
$$\frac{\alpha^X_1}{\alpha^X_0} = \frac{\alpha^X_1}{1 - \alpha^X_1}
\quad \text{and} \quad \frac{\alpha^Y_1}{\alpha^Y_0} = \frac{\alpha^Y_1}{1 -
  \alpha^Y_1}$$ 
are $4\epsilon$-close, and thus we have that the base
rates $\alpha^X_1$, $\alpha^Y_1$ must be $4\epsilon$-close (and the
same for $\alpha^X_0$, $\alpha^Y_0$).

But there is a catch. We can only apply the above identity when it
is  well-defined---that is,
%andrew7
%no divisions
%raf8
%when it does not divide by zero. 
when there are no divisions by zero.
In other words, we
can only apply it if there exists some outcome $o$ such that
$$\pr{\C(O(\bsigma)) = o \wedge f(\bsigma) =
  0} > 0 \quad \text{and} \quad \pr{\C(O(\bsigma)) = o \wedge f(\bsigma) = 1} > 0.$$
If there is no such outcome, $\C$ \emph{perfectly distinguishes} between the two classes, and thus
$$\{\O(\bsigma) \mid f(\bsigma) = 0\} \quad \text{and} \quad \{\O(\bsigma) \mid f(\bsigma) =
1\}$$ must have disjoint support. Hence,
in either case, $\P$ is a $4\epsilon$-trivial context.

%andrew7: cut "Triviality of", makes it seem like achieving fairness is trivial, rather than impossible
\paragraph{Simultaneously achieving fair treatment and predictive parity (general contexts).}
%\subsection{Proof Outline for the General Case}
Dealing with non-binary contexts is quite a bit more involved, and we
content ourselves to simply provide a very high-level overview.
%The full proof of Theorem \ref{thm:informal2} follows from the same observations about predictive parity and the same equation relating treatment and predictivity rates to base rates; however, for non-binary $\Class_\P$, the argument is somewhat more complex.
Consider some $\C$ that satisfies $\epsilon$-fair treatment and
$\epsilon$-predictive parity with respect to $\P = (\D, f, g, O)$; our
goal is again to show that $\P$ must be $O(\epsilon)$-trivial.

At a high level, we will show either that base rates are
$\epsilon$-close or that we can split the set of classes $\Class_{\P}$
into \emph{proper} subsets $\Class_1, \Class_2$ such that the classifier can
perfectly distinguish between these sets of classes.
%andrew7
%raf8: andrew, please always comment out the old text, now i don't
%know what you changed, but i don't think the text below is correct...
Once we have
shown this property, we can next 
%inductively rely on it to prove the
repeatedly rely on it to prove the
theorem (more precisely, by recursively splitting up
%raf8: added either
%$\Class_1$ and $\Class_2$ and applying the same result; formally doing
either $\Class_1$ or $\Class_2$ and applying the same result; formally doing 
this turns out to be somewhat subtle.)

To prove the above property, our goal is to use the same high-level
approach as in the binary case. Assume that there do not exist
$\Class_1$ and $\Class_2$ such that $\C$ can perfectly distinguish
between them, and let us show that then the base rates must be close.
%raf8
%The problem that arises in trying to show this is that
%since classes are not binary, we need to show that, 
In order to apply the same argument as in the binary case, we would
need to show that 
for \emph{all} pairs of
classes $(i,j)$, the above identity can be applied. If we do this,
then we have that, for \emph{all} $(i,j)$, the ratios
$$\frac{\pr{f(\bsigma_X) = j }}{\pr{f(\bsigma_X) = i }} \quad
\text{and} \quad \frac{\pr{f(\bsigma_Y) = j }}{\pr{f(\bsigma_Y) = i }}$$
are $4\epsilon$-close, from which we can conclude that the base rates
are $4\epsilon$-close. However, the fact that $\C$ cannot
distinguish between two proper subsets of classes \emph{does not} mean that all
%andrew7: write out w.r.t.
classes are ``ambiguous'' with respect to $\C$ (in the sense that $\C$ cannot
%raf8:
%perfectly tell them apart, so that we may directly apply the
%identity above). 
perfectly tell them apart, and thus the identity is well-defined). %, so that we may directly apply the
%andrew7: reworded for better flow
%raf8: andrew, again, i don't knwo what was there before.
Instead, what we show is that, under the
assumption that there do not exist two proper subsets of classes between which
$\C$ can perfectly distinguish, we have that, between \emph{any two
classes} $i$ and $j$, there exists a sequence of classes $(i_1, \ldots, i_n)$
such that $n \leq k$ ($k$ being the number of classes), $i_1 = i, i_n = j$, and any two \emph{consecutive} classes must be
``ambiguous''. Ambiguity
between classes with respect to $\C$ turns out to be exactly the condition
under which the above identity is well defined.
At a very high level, we can then perform a ``hybrid
argument'' over the classes in the
sequence to still conclude that, for all pairs of classes $(i,j)$, the ratios
$$\frac{\pr{f(\bsigma_X) = j }}{\pr{f(\bsigma_X) = i }} \quad
\text{and} \quad \frac{\pr{f(\bsigma_Y) = j }}{\pr{f(\bsigma_Y) = i }}$$
are $4(k-1)\epsilon$-close; this suffices to conclude that the base
rates between groups are close. 

\section{Preliminaries and Definitions}
\label{sec:definitions}

%raf2
\subsection{Notation}

\paragraph{Conditional probabilities.}
Given some random variable $X$ and some event $E$, we let $\Pr[p(X)
\mid E]$ denote the probability of a predicate $p(X)$ holding when conditioning the
probability space on the event $E$. If the probability of $E$ is 0, we
slightly abuse notation and simply define $\Pr[p(X) \mid E] = 0$.

\paragraph{Multiplicative distance.}
The following definition of multiplicative distance will be useful to
us.
%\begin{definition}
We let the \define{multiplicative distance} $\mu(x,y)$ between two
real numbers $x, y \geq 0$ be defined as 

\[
\md(x,y) = 
\begin{cases}
    \text{ln}
\left( \text{max} \left( \frac{x}{y}, \frac{y}{x} \right) \right)				& \text{if } x > 0, y > 0\\
	0 & \text{if } x = y = 0 \\
    \infty              & \text{otherwise}
\end{cases}
\]

\subsection{Classification Context}
We start by defining classification contexts and classifiers.
%raf2: can you define a macro \define that we can use for all defs.
\begin{definition}
A \define{classification context} $\P$ is denoted by a tuple $(\D, f, g, O)$ such that:
\begin{itemize}
\item $\D$ is a probability distribution with some finite support $\Sigma_\P$ (the set of all possible \define{individuals} to classify).
\item $f : \Sigma_\P \rightarrow \Class_\P$ is a surjective function that maps each individual to their \define{class} in a set $\Class_\P$.
\item $g : \Sigma_\P \rightarrow \G_\P$ is a surjective function that maps each individual to their \define{group} in a set $\G_\P$.
%raf2
\item $O : \Sigma_\P \rightarrow \{0,1\}^*$ is a function that maps
  each individual to their \define{observable features}.%
% in a set $\F_\P$.
\footnote{This is included for generality; for our result, it suffices to take $O$ to be the identity function, as we can show impossibility even for classifiers which may observe every feature of an individual.}
\end{itemize}
\end{definition}
We note that $f$ and $g$ are deterministic; this is without loss of
generality as we can encode any probabilistic features that $f$ and
$g$ may depend on into $\sigma$ as ``unobservable features'' of the
individual. 

%raf2: took from intro.
Given such a classification context
$\P$, we let $\Class_{\P}$ denote the range of $f$, and $\G_{\P}$
denote the range of $g$. Whenever the classification context $\P$ is
clear from context, we drop the subscript; 
additionally, whenever the distribution $\D$ and group function $g$ are clear from context, we
use $\bsigma$ to denote a random variable that is distributed
according to $\D$, and $\bsigma_X$ to denote the random variable
distributed according to $\D$ conditioned on $g(\sigma) = X$.

%Furthermore, the assumptions that $f$ and $g$ are surjective are also without loss of generality, as we may otherwise restrict $\Class%_\P$ and $\G_\P$ to the respective ranges of $f$ and $g$.

%Next we can define a classifier for a context as follows:
%raf
%\begin{definition}
A \define{classifier} $\C$ for a classification context $\P = (\D, f,
g, O)$ is simply a (possibly randomized) algorithm. We let $\Out^{\C}_{\P}$ denote the support of the distribution $\{\C(\bsigma)\}$.
\subsection{Fair Treatment}

Next we define the notion of \emph{$\fairtreat$}, an approximate
version of the notion of ``equalized odds" from Hardt et al.
\cite{hardt} (which in turn was derived from notions implicit in the ProPublica study
\cite{compas}). 

%raf2: new
\begin{definition}
We say that a classifier $\C$ satisfies
\define{$\epsilon$-$\fairtreat$} with respect to a context $\P = (\D,
f, g, O)$ if, for any groups $X, Y \in \G_\P$, any class $c \in
\Class_{\P}$, and any outcome $o \in \Out_\P^\C$, we have that
%the multiplicative distance
  $$\mu(\Pr [\C(\O(\bsigma_X)) = o \mid f(\bsigma_X) = c], \Pr [\C(\O(\bsigma_Y)) = o \mid f(\bsigma_Y) = c]) \leq \epsilon$$
\end{definition}
%In other words, 
%the max-divergence\footnote{
%Recall that the \textbf{max-divergence} between two probability distributions
%$\D_1$ and $\D_2$ is $\Div(\D_1 || \D_2) = \max_{o \in \text{sup}(\D_1) \cup
 % \text{sup}(\D_2)} \text{ln} (\Pr_{\D_1}[o] /  \Pr_{\D_2}[o])$. It is not hard to show that $$\max(\Div(\D_1 || \D_2), \Div(\D_2 || \D_1%)) = \max_{o \in \text{sup}(\D_1) \cup
 % \text{sup}(\D_2)} \mu(\Pr_{\D_1}[o], \Pr_{\D_2}[o])$$
%.} between $\{ \C(\O(\bsigma_X)) | f(\bsigma_X) = c \}$ and $\{
%\C(\O(\bsigma_Y)) | f(\bsigma_X) = c \}$ is at most $\epsilon$ for any $c \in
%\Class_{\P}$.

%raf2
\iffalse
%letting 
%$$\D_{X,c} = \lbrace \sigma \leftarrow \D | f(\sigma) = c \wedge g(\sigma) = X : \C(O(\sigma))\rbrace$$ (respectively for $Y$), it hol%ds for each $c \in \Class_\P$ that: 
\begin{itemize}
\item $\Pr[f(\bsigma_X) =c]$
$\D_{X, c}$ and $\D_{Y, c}$ are well-defined (i.e., $\prc{f(\sigma) = c}{X} > 0$ and $\prc{f(\sigma) = c}{Y} > 0$).\footnote{Notice that, by assumed surjectivity of $f$, $\prc{f(\sigma) = c}{X} > 0$ for at least one $g \in \G_\P$, so in order for taking the max-divergence to make sense, it should hold for all $g$.}
\item $\Div(\D_{X, c} || \D_{Y, c}) \leq \epsilon$
\end{itemize}
\end{definition}
\fi
Note that for the case of binary classification tasks and binary
classifiers (i.e., when $\Class_\P = \Out^{\C}_\P = \Bit$), fair treatment
is equivalent to requiring ``similar" false positive and false negative
rates.

\subsection{Rational Fairness}
We turn to introducing our notion of ``fairness with respect to rational decision-makers''.
Towards this goal, given a classification context $\P$ and a
classifier $\C$, we consider a single-player Bayesian game $\Gamma$ where
individuals $\sigma$ are sampled from $\D$, the decision-maker (DM) gets to see the
group $g(\sigma)$ of the individual and
the outcome $o = \C(\sigma)$ of the classifier, and then selects some action $x
\in \Dec$. They then receive utility
$u(f(\sigma), x)$ that
is only a function of the actual class $f(\sigma)$ and their decision
$x$. 
We let $\Gamma^{\P,  \C, \Dec, u}$ denote the Bayesian game induced by
the above process (for some action space $\Dec$ and utility function
$u$). Given such a game $\Gamma^{\P,  \C, \Dec, u}$, a \define{pure strategy
  for the DM} is a function $s : \G_{\P} \times \Out^{\C}_{\P}
\rightarrow \Dec$, and a \define{mixed strategy} is a probability
distribution over pure strategies. In the sequel, we simply use the
term ``strategy" to refer to mixed strategies.

\begin{definition}
We say that a strategy $s$ is \define{$\epsilon$-optimal} in
$\Gamma^{\P,  \C, \Dec, u}$ where $\P = (\D,
f, g, O)$,
if for all $(g,o)$ in the support of $\{ (g(\bsigma)), \C(\O(\bsigma))
\}$ and any strategy $s'$, we have:
%raf10
%$$\E [ u(f(\bsigma)), s( g, o) | g(\bsigma) = g, \C(\bsigma)) = o] \leq 
%e^{\epsilon} \E [ u(f(\bsigma), s'( g, o)) | g(\bsigma) = g,
%\C(\sigma)) = o]$$
$$e^{\epsilon} \E [ u(f(\bsigma)), s( g, o) \mid g(\bsigma) = g, \C(\bsigma) = o] \geq 
\E [ u(f(\bsigma), s'( g, o)) \mid g(\bsigma) = g, \C(\bsigma) = o]$$
\end{definition}
That is, a player can never gain more than a factor $e^\epsilon$ in utility
by deviating; in other words, $s$ is an $\epsilon$-Nash equilibrium in  $\Gamma^{\P,  \C, \Dec, u}$.\footnote{Note that
  we here use the so-called \emph{ex-interim} notion of an $\epsilon$-Nash
  equilibrium which requires $s$ to be $\epsilon$-close to the optimal
  strategy even conditioned on the DM having received its type (i.e.
  $(g,o)$ in our case). This is the most commonly used notion of an
  $\epsilon$-Nash equilibrium. We mention that there is also a weaker notion of \emph{ex-ante}
  $\epsilon$-Nash equilibrium which only requires $s$ to be optimal \emph{a priori}
  before seeing the type. A weaker version of our main impossibility
result holds also for this notion.}
We turn to defining what it means for a strategy to be fair.
\begin{definition}
We say that a strategy $s$ for a game  $\Gamma^{\P,  \C, \Dec, u}$ is
\define{fair} if there exists a function $\tilde{s}$ such that
$s(g,o)=\tilde{s}(o)$.
\end{definition}
That is, the strategy $s$ does not depend on the group of the
individual. 
%raf4
As we shall see later on (see Claim \ref{clm:RF.FP}),
the ``best'' fair strategy $s$ (i.e., a fair strategy that satisfies
$\epsilon$-optimality for the smallest $\epsilon$) may need to be a
mixed strategy---in fact, we demonstrate a game where there is a
significant gap between the best mixed and pure fair strategies. (In
our opinion, this is intriguing in its own right, as mixed strategies
are typically not helpful in a decision-theoretic---i.e.,
single-player---setting.)

We finally define what it means for a classifier $\C$ to enable fair decision
making.

\begin{definition}
We say that $\C$ \define{enables $\epsilon$-approximately fair decision making} (or simply
satisfies \define{$\epsilon$-rational fairness}) with respect to the context
$\P = (\D, f, g, O)$ if, for every 
finite action space $\Dec$ and every
utility function $u:\Class_{\P} \times \Dec \rightarrow [0,1]$, 
there exists a strategy $s$ that is fair and 
$\epsilon$-optimal with respect to
$\Gamma^{\P, \C, \Dec, u}$.
\end{definition}

\section{Characterizing Fair Classifiers}
Our main theorem is a complete characterization of the class of contexts that
admit classifiers that simultaneously satisfy fair treatment and
rational fairness.

The following notion of ``triviality'' will characterize contexts
admitting such classifiers.

\begin{definition}
A classification context $\P = (\D, f, g, O)$ is
\textbf{$\epsilon$-trivial} if there exists a partition of the set
$\Class_{\P}$ into 
%raf12: not proper here
%disjoint proper 
subsets $\Class_1, \Class_2,
\ldots, \Class_m$ of classes such that the following conditions hold:
\begin{itemize}
\item[(1)] For any $i \in [m]$, 
%raf10: i didn't notice this before! both are fine but the latter it
%more natural and less divisions by 0
%$c \in \Class_\P$ and any two groups
$c \in \Class_i$ and any two groups
  $X,Y$ in $\G_{\P}$, we have that $$\md(\pr{ f(\bsigma_X) = c \mid
    f(\bsigma_X) \in \Class_i }, \pr{ f(\bsigma_Y) = c \mid f(\bsigma_Y)
    \in \Class_i }) \leq \epsilon$$ (i.e., the \emph{base rates}
  conditioned on $\Class_i$ are close between groups)
\item[(2)] For any $i,j \in [m]$ with $i \neq j$, the distributions $\braces{\O(\bsigma) \mid f(\bsigma) \in \Class_i}$ and $\braces{\O(\bsigma) \mid f(\bsigma) \in \Class_j}$ have disjoint support.
\end{itemize}
\label{def:trivial}
\end{definition}
Note that if the class space $\Class$ is binary, triviality means that
either the base rates are $\epsilon$-close, or we can perfectly
distinguish between the two classes.

Our main characterization theorem shows that a context $\P$ admits
classifiers satisfying $O(\epsilon)$-fair treatment and
$\epsilon$-rational fairness \emph{if, and only if}, $\P$ is $O(\epsilon)$-trivial.

\begin{theorem}
[Theorem 2, restated]
Consider some classification context $\P = (\D, f, g, O)$ and let $k = |\Class_\P|$ (i.e., the
number of classes). Then:
\begin{enumerate}
\item[(1)] For any constant $\epsilon$, if $\P$ is $\epsilon$-trivial,
  then there exists a classifier $\C$ satisfying 0-fair treatment and $2\epsilon$-rational fairness with respect to $\P$.
\item[(2)] For any constant $\epsilon < 3/2$, if there exists a classifier $\C$ satisfying
  $\epsilon$-fair treatment and $\epsilon / 5$-rational fairness with
  respect to $\P$, then $\P$ is $4(k-1)\epsilon$-trivial.
\end{enumerate}
\label{thm:impossibility}
\end{theorem}

%raf7: not so interesting
\iffalse
We note that the $\epsilon / 5$ bound for rational fairness is somewhat of a simplification; a more precise and tighter bound is in fact $\text{log}\left( \frac{2}{1 + e^{-\epsilon/2}} \right)$ (see Claim \ref{clm:RF.FP}), which is approximately $\epsilon / 4$ for $\epsilon \approx 0$, but is asymptotically less than $\epsilon / c$ for any $c$ as $\epsilon \rightarrow \infty$. Generally, however, any classifier with $\epsilon > 3/2$ would have egregious fairness error, so we stress that our assumption is reasonable.
\fi
%raf7:just bragging :-)
\iffalse
This amounts to an almost-tight characterization of the possible contexts $\P$ in which we can achieve both of these intuitive desiderata for fairness. In particular, this is a fairly striking result, as, in most realistic scenarios, we lack full distinguishability between any subgroups of classes, and so our theorem suggests that the error in fairness will always be linearly proportional to the discrepancy in base rates!
\fi
Note that Theorem \ref{thm:informal} from the introduction (i.e., the
classification for binary contexts) follows directly as a special case
when $k = 2$.
%; we present a separate proof overview in the introduction as many of the more complex arguments we require for $k > 2$ are unnecessary or follow trivially in this special case, but here we prove only the more general version.
%
Additionally, let us remark that Theorem \ref{thm:impossibility} holds even for a
somewhat weaker definition of rational fairness where we only require
the existence of a fair $\epsilon$-equilibrium in games with
\emph{binary decision spaces} (i.e., $\Dec = \Bit$), and even
if we restrict to this simple and natural subclass of games.
%; this extension is trivial as all of our rational-fairness-related claims (Claims \ref{clm:FP.RF} and \ref{clm:RF.FP}) are proven using (or %trivially extend to) binary decision spaces.

\section{Proof of Theorem \ref{thm:impossibility}}
\label{sec:thm1}
Towards proving Theorem \ref{thm:impossibility}, we 
%raf10
first define a notion of $\epsilon$-predictive parity and show that 
a context $\P$ admits classifiers satisfying 
$\epsilon$-rational fairness and $\epsilon$-fair treatment \emph{if
  and only if} $\P$
admits a classifier satisfying $O(\epsilon)$-predictive parity and
$\epsilon$-fair treatment. (We note that predictive parity is \emph{not}
equivalent to rational fairness, but is so for classifiers that also
satisfy fair treatment.) 

Next, we show that $O(\epsilon)$-triviality
characterizes the set of contexts admitting classifiers satisfying
$\epsilon$-predictive parity and $\epsilon$-rational fairness.
(This second step is interesting in its own right, and can be thought
of a significant strengthening of the impossibility result of
Chouldechova \cite{chouldechova}, which
only showed triviality for the special case when $\epsilon = 0$, the classification
context is binary, and the classifier is binary.\footnote{For this
  special case, her notion of triviality is a special case of our
  notion of $0$-triviality, which requires that either the base rates
  are identical for both groups, or one can perfectly predict the
  class of an individual.})
%devote the first few subsections to presenting several important definitions and claims that will be necessary in the body of the proof. Then, in Section \ref{sec:proof1}, we present the primary lemma used to prove the theorem, and prove the lemma in Section \ref{sec:proof2}.

\subsection{Predictive Parity}
We first introduce an
intermediate notion of approximate ``predictive parity" (we are borrowing the name
%raf4* is it true that she used this name
from Chouldechova in \cite{chouldechova}, who considered a perfect
version of this notion tailored for binary classifiers, where the
class is a single bit and the classifier also outputs only a single bit.)
Roughly speaking, $\epsilon$-predictive parity requires
that the distributions of individuals' \emph{classes}, conditioned on a
particular \emph{outcome}, be $\epsilon$-close between groups.
%raf4 is it strict relaxation of what kleinberg consideers
We remark that this notion is a strict relaxation of the
notion of $\epsilon$-group calibration considered by Kleinberg et al.
\cite{kleinberg} (which not only requires that the distribution of the
classes be the same between groups conditioned on the outcome $o$ of
the classifier, but
also that the outcome $o$ ``accurately predicts'' the class).

\begin{definition}
We say that a classifier $\C$ satisfies $\epsilon$-\textbf{predictive
  parity} with respect to a context $\P = (\D, f, g, O)$ if, for any
groups $X, Y \in \G_\P$, any outcome $o \in \Out_\P^\C$, and any class $c \in
\Class_{\P}$, we have that
%the multiplicative distance
 $$\md(\pr{ f(\bsigma_X) = c \mid  \C(O(\bsigma_X)) = o}, \pr{ f(\bsigma_Y) = c \mid  \C(O(\bsigma_Y)) = o}) \leq \epsilon$$ 
 \end{definition}
 
% In other words, 
%the max-divergence between $\{ f(\bsigma_X) | \C(\O(\bsigma_X)) = o \}$ and $\{f
%(\bsigma_Y) | \C(\O(\bsigma_Y)) = o \}$ is at most $\epsilon$ for any $o \in \Out_\P^\C$.

We next show that predictive parity is closely related to rational
fairness (at least, when combined with fair treatment).
We first show that $\epsilon$-predictive parity implies $O(\epsilon)$-rational fairness.

\begin{claim}
%raf4: rewrote
%Any classifier $\C$ satisfying $\epsilon$-predictive parity with
%respect to a context $\P = (\D, f, g, O)$ will satisfy
%$2\epsilon$-rational fairness with respect to $\P$.
Let $\C$ be a classifier that satisfies $\epsilon$-predictive parity with
respect to a context $\P = (\D, f, g, O)$. Then $\C$ satisfies $2\epsilon$-rational fairness with respect to $\P$.
\label{clm:FP.RF}
\end{claim}

\begin{proof}
%raf3: added
Consider some classifier $\C$ satisfying $\epsilon$-predictive parity with
respect to a context $\P$. We will show that $\C$ also
satisfies $2\epsilon$-rational fairness with respect to $\P$.

%raf3:
%Let the support of  $\{ (g(\bsigma)), \C(\O(\bsigma)) \}$ henceforth
%be denoted by $\T^\Gamma_{DM}$ for notational simplicity. 
%raf3: added
%raf4*: andrew, shouldn't we get rid of this?? did you check if it was
%needed?
\iffalse
Let $\T^\Gamma_{DM}$ denote the support of $\{ (g(\bsigma)),
\C(\O(\bsigma)) \}$ (i.e., the support of the type space in the
Bayesian game). Let us first argue that $\$\T^\Gamma_{DM}
= \Out_\P^\C \times \G_\P$. Assume for contradiction that 
$\$\T^\Gamma_{DM} \neq \Out_\P^\C \times \G_\P$. Then, since
$\Out_\P^\C$ is the support of $\C(O(\bsigma))$, there must exists
some $o \in \Out_\P^\C$
%First,
%observe that $\C$ cannot satisfy predictive parity if $\T^\Gamma_{DM}
%\neq \Out_\P^\C \times \G_\P$; in this case, since $\Out_\P^\C$ is
%the support of $\C(O(\bsigma))$ (i.e., for every $o \in \Out_\P^\C$
%there is at least one $g \in \G_\P$ for which $(o, g) \in
%\T^\Gamma_{DM}$), there must exist an outcome $o \in \Out_\P^\C$ 
and
groups $g, g' \in \G_\P$ such that $(o, g) \in \T^\Gamma_{DM}$ but
$(o, g') \not\in \T^\Gamma_{DM}$, 
%raf3
%and so predictive parity between
%$g$ and $g'$ is impossible from the definition.
which contradicts predictive-parity.
\fi
Let $\T^\Gamma_{DM}$ denote the support of $\{ (g(\bsigma)),
\C(\O(\bsigma)) \}$ (i.e., the support of the type space in the
Bayesian game). 
%raf10:
Let $s^*$ be an optimal strategy; we may without loss of generality
assume that $s^*$ is a pure strategy, since for every type $(g,o) \in
\T^\Gamma_{DM}$ there exists a deterministic best response.
\iffalse
Consider the strategy which is optimal for every type; call it
$s^*$. In particular, notice that $s^*$ can always be a pure strategy
(i.e., $s : \T^\Gamma_{DM} \rightarrow \Dec$), since it may assign to
each pair $(o, g)$ any decision $x \in \Dec$ which maximizes the
quantity
$$\sum_{y \in \Class_\P} \pr{f(\bsigma_g) = y \mid \C(O(\bsigma_g)) = o} u(y, x)$$
and the distribution of $f(\sigma)$ conditioned on $\C(O(\sigma)) = o$ is \emph{a priori} determined by $\P$ and $\C$.
\fi
%raf10
We now show how to modify $s^*$ into a fair strategy $s$ without ever
losing too much in expected utility.
%from 
%Assume that there exist types $(o, g_1), (o, g_2) \in \T^\Gamma_{DM}$ for which $s^*(o, g_1) \neq s^*(o, g_2)$; otherwise we are done by taking $s = s^*$. In this case, fix any such $o$ and pick a type $(o, g^*_o)$ arbitrarily from $\T^\Gamma_{DM}$; we will show that, for every pair $(o, g') \in \T^\Gamma_{DM}$, the utility of the fair strategy $s(o, g') \triangleq s^*(o, g^*_o)$ gives at least $e^{-2\epsilon}$ times as much utility as the optimal strategy $s^*(o, g')$ in expectation. Notice that, if this is the case, then the fair strategy $s$ which we define here will by definition be a $2\epsilon$-approximate Nash equilibrium for the game, completing the proof.
For every outcome $o$, pick some $g^*_o$ such that $(g^*_o,o) \in
\T^\Gamma_{DM}$, and define $s(g,o)= s^*(g^*_o,o)$. Clearly $s$ is
fair. We now show that for every pair $(g,o) \in \T^\Gamma_{DM}$, the
expected utility of playing $s^*$ can never be more than a factor
$e^{2\epsilon}$ better than the expected utility of playing $s$, and
thus $s$ is $2\epsilon$-optimal (as desired).

Assume for contradiction that there exists some $(g,o)$ such that $g \neq g_o^*$ and
$$\E [ u(f(\bsigma), s^*( g, o)) \mid g(\bsigma) = g, \C(O(\sigma)) = o]>
e^{2\epsilon} \E [ u(f(\bsigma), s( g, o) \mid g(\bsigma) = g,
\C(O(\bsigma)) = o]$$
That is, 
$$\sum_{y \in \Class_\P} \pr{f(\bsigma_{g}) = y \mid \C(O(\bsigma_{g})) =
    o}u(y, s^*(g,o)) > e^{2\epsilon} \sum_{y \in \Class_\P}
  \pr{f(\bsigma_{g}) = y \mid \C(O(\bsigma_{g})) = o} u(y, s(g,o)) $$
By applying predictive parity (more precisely, that the multiplicative
distance between $\pr{f(\bsigma_{g}) = y \mid \C(O(\bsigma_{g})) =
    o}$ and $\pr{f(\bsigma_{g^*_o}) = y \mid \C(O(\bsigma_{g^*_o})) =
    o}$ is at most $\epsilon$) to both the LHS and the RHS (we lose a
factor $e^\epsilon$ for each application), we get that 
$$\sum_{y \in \Class_\P} \pr{f(\bsigma_{g^*_o}) = y \mid \C(O(\bsigma_{g^*_o})) =
    o} u(y, s^*(g,o)) > \sum_{y \in \Class_\P}
  \pr{f(\bsigma_{g^*_o}) = y \mid \C(O(\bsigma_{g^*_o})) = o} u(y,
  s(g,o))$$
In other words, (and relying on the fact that $s(g,o) = s^*(g^*_o,o)$), 
$$\E [ u(f(\bsigma), s^*( g, o)) \mid g(\bsigma) = g^*_o, \C(O(\sigma)) = o]>
\E [ u(f(\bsigma), s^*( g^*_o, o) \mid g(\bsigma) = g^*_o,
\C(O(\bsigma)) = o]$$
which is a contradiction since, by assumption, $s^*(g_o^*,o)$ is an optimal move given
the type $(g^*_o,o)$.
\end{proof}

%raf4
As we next show, any classifier that satisfies $\epsilon$-rational
fairness \emph{and} $\epsilon'$-fair treatment (for any $\epsilon'$)
also satisfies $O(\epsilon)$-predictive parity.
%  if we have fair treatment (ensuring that every group can potentially receive each possible outcome), also implies predictive parity by the following claim of the contrapositive. 
%
%raf4: rewrote below
Intuitively, we show this as follows. Consider some $\C$ that does not
satisfy $O(\epsilon)$-predictive parity, yet satisfies $\epsilon$-rational
fairness and $\epsilon'$-fair treatment. This means there exists some
class $y^*$, groups $g,g'$ and some outcome $o$ such that the
prevalence of $y^*$  is significantly higher in group $g$ than in
group $g'$ conditioned on the classifier outputting $o$. 

We then
construct a very natural game for the DM where every fair strategy has
low utility compared to the optimal unfair strategy. The action
space consists of two actions $\{\Risky, \Safe\}$. If the DM chooses
$\Safe$ they always receive some fixed utility $u^*$. On the other hand, if they
choose $\Risky$, they receive 1 if the individual's class is $y^*$ and 0
otherwise. That is, playing $\Risky$ is good if the individual is
``good'' (i.e., in class $y^*$) and otherwise not. 

Assume there exists some fair strategy $s$ that is $\epsilon$-optimal
in this game.
We first observe that by $\epsilon'$-fair treatment of $\C$, it must
be the case that both $(g,o)$ and $(g,o')$ are in the support of $\{
(g(\bsigma)), \C(\O(\bsigma)) \}$ (i.e., the support of the ``type
distribution'' of the game), and thus optimality of $s$ must
hold conditioned on both of them.

We next use the fact that the prevalence of $y^*$ is significantly
higher conditioned on the DM getting the signal $(o,g)$ than when
getting $(o,g')$, and thus if we set $u^*$ (i.e, the utility of
playing $\Safe$) appropriately, we can
ensure that the DM gains by discriminating between
$g$ and $g'$ (playing $\Risky$ when the group is $g$, and $\Safe$
otherwise). Interestingly, determining by how much a DM can gain by
discriminating turns out to be somewhat subtle; it turns out that the ``best'' fair
strategy (i.e., the fair strategy that minimizes the
expected utility loss with respect to the optimal strategy) mixes with
probability 1/2 between $\Risky$ and $\Safe$.

\begin{claim}
%raf4: rewrote
Let $\C$ be a classifier that satisfies $\text{log}\left( \frac{2}{1 + e^{-\epsilon/2}} \right)$-rational fairness with
respect to a context $\P = (\D, f, g, O)$, as well as $\epsilon'$-fair
treatment with respect to $\P$ (for any $\epsilon'$). Then $\C$ satisfies
$\epsilon$-predictive parity with respect to $\P$.
\label{clm:RF.FP}
\end{claim}

\begin{proof}
%raf4:
%We prove the contrapositive: any classifier $\C$ \emph{not} satisfying $\epsilon$-predictive parity and $\delta$-fair treatment with respect to $\P$ will \emph{not} satisfy $\epsilon/4$-rational fairness with respect to $\P$.
Assume for contradiction that $\C$ satisfies $\text{log}\left( \frac{2}{1 + e^{-\epsilon/2}} \right)$-rational
fairness and $\epsilon'$-fair treatment (with respect to $\P$), yet does not satisfy $\epsilon$-predictive
parity (with respect to $\P$).

Let $\T^\Gamma_{DM}$ denote the support
of $\{ (g(\bsigma)), \C(\O(\bsigma)) \}$ . We first claim that $\T^\Gamma_{DM}
= \Out_\P^\C \times \G_\P$. If not, since $\Out_\P^\C$ is the support
of $\C(O(\bsigma))$ (and thus for every $o \in \Out_\P^\C$ there is at
least one $g \in \G_\P$ for which $(o, g) \in \T^\Gamma_{DM}$), 
there
must exist an outcome $o \in \Out_\P^\C$ and groups $g, g' \in \G_\P$
such that $(o, g) \in \T^\Gamma_{DM}$ but $(o, g') \not\in
\T^\Gamma_{DM}$. This, however, would mean that there is $y \in
\Class_\P$ for which the distributions 
%raf4*:
%$\lbrace \C(O(\sigma)) |
%f(\sigma) = y \rbrace_{g}$ and $\lbrace \C(O(\sigma)) | f(\sigma) = y
%\rbrace_{g'}$ have different supports (and hence infinite
$\lbrace \C(O(\bsigma_g)) \mid
f(\bsigma_g) = y \rbrace$ and $\lbrace \C(O(\bsigma_{g'})) \mid f(\bsigma_{g'}) = y
\rbrace$ have different supports (and hence infinite
max-divergence), as, by definition of $\T^\Gamma_{DM}$, $o$ must be in
the support of the former for some $y$ but cannot be in the support of
the latter for any $y$. This contradicts  $\epsilon'$-fair treatment
(for any $\epsilon'$) of $\C$.

Next, since $\C$ fails to satisfy $\epsilon$-predictive parity, there
exist groups $g, g' \in \G_\P$, class $y^* \in \Class_\P$, and an outcome
$o \in \Out_\P^\C$ such that 
%raf5*: this is where you are not explicit which complicates the proof
%and harms readability
%there exists a $p$ for which:
%$$e^{-\epsilon/2} \pr{f(\bsigma_g) = y^* | \C(O(\bsigma_g)) = o} > p > e^{\epsilon / 2} \pr{f(\bsigma_{g'}) = y^* | \C(O(\bsigma_{g'})) = o}$$
$$\frac{\pr{f(\bsigma_g) = y^* \mid \C(O(\bsigma_g)) =
    o}}{\pr{f(\bsigma_{g'}) = y^* \mid \C(O(\bsigma_{g'})) = o}} > e^{\epsilon}$$
Let $\delta > \epsilon$ be such that 
$$e^{\delta} = \frac{\pr{f(\bsigma_g) = y^* \mid \C(O(\bsigma_g)) =
    o}}{\pr{f(\bsigma_{g'}) = y^* \mid \C(O(\bsigma_{g'})) = o}}$$
and define the ``midpoint'' $p$ between these probabilities as
$$p = e^{\delta/2} \pr{f(\bsigma_{g'}) = y^* \mid \C(O(\bsigma_{g'})) =
  o} = e^{-\delta/2} \pr{f(\bsigma_g) = y^* \mid \C(O(\bsigma_g)) =
    o}$$
Consider a game where $\Dec = \braces{\Risky, \Safe}$, $u(y, \Risky) =
1$ if $y = y^*$ and $0$ otherwise and $u(\cdot, \Safe) = p$. 
The decision-maker's expected utility for choosing $\Safe$ is always $p$; on the other
hand:
\begin{itemize}
\item Conditioned on $(o, g)$, the decision-maker's expected utility
  for choosing $\Risky$ is $$\pr{f(\bsigma_g) = y^* \mid \C(O(\bsigma_g))
%raf5
%    = o} > e^{\epsilon / 2} p$$
    = o} = e^{\delta / 2} p$$
\item Conditioned on $(o, g')$, their expected utility for $\Risky$ is 
%raf5
%$$\pr{f(\bsigma_{g'}) = y^* | \C(O(\bsigma_{g'})) = o} < e^{-\epsilon
%  / 2} p$$
$$\pr{f(\bsigma_{g'}) = y^* \mid \C(O(\bsigma_{g'})) = o} = e^{-\delta / 2} p$$
\end{itemize}
Consider some fair \emph{pure} strategy $s$ for DM. It must choose either $\Risky$ or
$\Safe$ for both $(o, g)$ and $(o, g')$. So,
\begin{itemize}
%raf5
%\item If $s$ chooses $\Risky$, DM receives at most $e^{-\epsilon / 2}
%  p$ in expected
\item If $s$ chooses $\Risky$, DM receives $e^{-\delta / 2} p$ in expected
utility conditioned on $(o, g')$, whereas they could have received $p$
%raf5
%by instead choosing $\Safe$ (thus, they incur a multiplicative loss of
%$e^{\epsilon / 2}$).
by instead choosing $\Safe$ (thus, they incur a multiplicative loss of $e^{\delta / 2}$).
\item On the other hand, if $s$ chooses $\Safe$, then the
  decision-maker receives utility $p$ conditioned on $(o, g)$, whereas
%  they could have received at least $e^{\epsilon/2}p$ utility in
  they could have received $e^{\delta/2}p$ utility in
  expectation by instead choosing $\Risky$ (again incurring a multiplicative loss of $e^{\delta / 2}$).
\end{itemize}
Thus, we conclude that any fair pure strategy must lose at least a
%raf5
$e^{\delta / 2}$ multiplicative factor in utility compared to the
optimal (unfair) strategy (which chooses $\Risky$ for $(o, g)$ and $\Safe$ for $(o, g')$).

%raf4*: here is how i would do it...
Consider next a fair \emph{mixed} strategy $s$ that chooses $\Risky$
with probability $p_r$ (and $\Safe$ with probability $1-p_r$) for both
$(o,g)$ and $(o,g')$. 
\begin{itemize}
%raf5
\item Conditioned on $(o, g')$, the decision-maker gets $p_r e^{-\delta / 2} p + (1-p_r)p$
in expected utility, whereas they could have received $p$ by instead
choosing $\Safe$. This results in a multiplicative loss of $p_r e^{-\delta / 2}+ (1-p_r)$.
%raf5
%\item Conditioned on $(o, g)$, let $p_g = e^{\delta/2}p$ for some
%$\delta > \epsilon$; the decision-maker gets exactly $p_r e^{\delta /
%2} p + (1-p_r)p$ in expected utility
\item Conditioned on $(o, g)$, the decision-maker gets $p_r e^{\delta / 2} p + (1-p_r)p$ in expected utility,
whereas they could have received $e^{\delta/2}p$ utility in
  expectation by instead choosing $\Risky$, 
%raf5
yielding a multiplicative loss of $p_r + (1-p_r) e^{-\delta / 2}$.
%raf5.
\end{itemize}
Thus, when determining the rational strategy that minimizes the loss,
we may without loss of generality assume that $p_r \geq 1/2$ (as the
case when $p_r \leq 1/2$ is symmetric simply by renaming $p_r$ and
$1-p_r$), and focus on finding the $p_r$ that minimizes
$$p_r e^{-\delta / 2}+ (1-p_r) = 1 - p_r (1 - e^{-\delta / 2}) $$
which happens when $p_r$ is as small as possible, and thus 
when $p_r = 1/2$.
So the optimal mixed rational strategy must be $p_r = 1/2$, and thus has a
multiplicative loss of 
$$1/2 (e^{-\delta / 2}+ 1) <  1/2 (e^{-\epsilon / 2}+ 1)$$
for both $(o, g)$ and $(o, g')$ compared to the optimal unfair strategy. (In particular, using the expected utilities above, setting $p_r > 1/2$ worsens the multiplicative loss conditioned on $(o, g')$ by decreasing $p_r e^{-\delta / 2}+ (1-p_r)$, and setting $p_r < 1/2$ worsens the multiplicative loss conditioned on $(o, g)$ by decreasing $p_r + (1-p_r) e^{-\delta / 2}$.)

Hence, the decision-maker gains at least $\frac{2}{1 + e^{-\epsilon/2}}$ utility multiplicatively by switching from any fair mixed strategy to the optimal unfair strategy, and so we conclude the proof with the contradiction that $\C$ cannot satisfy $\text{log}\left( \frac{2}{1 + e^{-\epsilon/2}} \right)$-rational
fairness with respect to $\P$.

\iffalse
Hence, deviating from $\Safe$ to $\Risky$ for group $g$ will provide a multiplicative increase of $e^{\epsilon / 2}$ to the decision-maker's expected utility conditioned on $(o, g)$, as will deviating from $\Risky$ to $\Safe$ for $g'$ conditioning on $(o, g')$. So we can observe that the optimal strategy for the decision-maker is to make decision $\Risky$ for $g$ and $\Safe$ for $g'$, and that this strategy provides a multiplicative increase of more than $e^{\epsilon / 2}$ in expected utility, conditioned on either $(o, g)$ or $(o, g')$, compared to any other pure strategy (including all pure strategies independent of the individual's group).

Finally, to account for mixed strategies, we notice that the decision-maker deviating from any mixed strategy which offers the same distribution of decisions to $g$ as $g'$ to the optimal strategy above will result in a multiplicative gain of at least
$$1 + 1/2(e^{\epsilon / 2} - 1) > e^{\epsilon / 4}$$
in expected utility conditioned on either $(o, g)$ or $(o, g')$ (i.e., a multiplicative gain of $e^{\epsilon / 2}$ with probability 1/2), as either the probability of deciding $\Risky$ for $g$ must increase by at least 1/2 or the probability of deciding $\Safe$ for $g'$ must increase by at least 1/2. Hence no such mixed strategy can be an $\epsilon/4$-approximate Nash equilibrium.
\fi
\end{proof}

We note that, for $\epsilon < 3/2$ (in particular, $\epsilon$ less than roughly 1.644), we have that $\text{log}\left( \frac{2}{1 + e^{-\epsilon/2}} \right) > \epsilon / 5$, which demonstrates the bounds we show in our other results:

\begin{corollary}
Let $\epsilon \in (0, 3/2)$, and let $\C$ be a classifier that satisfies $\epsilon / 5$-rational fairness with
respect to a context $\P = (\D, f, g, O)$, as well as $\epsilon'$-fair
treatment with respect to $\P$ (for any $\epsilon'$). Then $\C$ satisfies
$\epsilon$-predictive parity with respect to $\P$.
\label{cor:RF.FP}
\end{corollary}

%raf7: not interesting.
%For very small $\epsilon \approx 0$, we note that $\text{log}\left( \frac{2}{1 + e^{-\epsilon/2}} \right) \approx \epsilon / 4$, but this bound is never sufficient for any $\epsilon$, so we use $\epsilon / 5$ as a crude approximation.

An interesting observation (which is not relevant for the sequel of
the proof, but nonetheless insightful) which follows from the above
proof is that 
the optimal fair strategy for the DM in the above
game is a \emph{mixed} strategy which uniformly mixes between
$\Safe$ or $\Risky$ (each with probability 1/2), whereas any fair pure
strategy loses a factor of $e^{\epsilon / 2}$ (i.e., significantly more)
in utility. (We note, however, that
the existence of such a gap between the fair mixed and fair pure strategies can only
arise in games where the optimal strategy is unfair: the existence of
an optimal fair mixed strategy implies the existence of an optimal
fair pure strategy.)

%raf8:
%\subsection{Trivial Contexts Imply Fairness}
\subsection{Proof of Theorem \ref{thm:impossibility} (1)}

By relying on the fact that predictive parity implies rational
fairness, we can now prove the first part of Theorem \ref{thm:impossibility}.

\begin{proposition}[\textit{Theorem \ref{thm:impossibility} (1).}]
If $\P = (\D, f, g, O)$ is an $\epsilon$-trivial context, then there exists a classifier $\C$ that satisfies 0-fair treatment and $2\epsilon$-rational fairness with respect to $\P$.
\label{clm:partOne}
\end{proposition}

\begin{proof}
%raf10:
%Let $\C : \F_\P \rightarrow [m]$ do as follows:
Consider a classifier $\C$ that on input $y$ (in the support of
$\O(\bsigma)$) recovers some $\sigma$ such that $\O(\sigma) = y$,
and then outputs $f(\sigma)$.
%raf10: sorry andrew, not precise enough below; you need to specify
%what C does...what does it mean use the disjoint distributions..
\iffalse
\begin{itemize}
\item Given $O(\sigma)$, use the disjoint distributions $\braces{\O(\bsigma) \mid f(\bsigma) \in
  \Class_i}$ to 
determine the (unique) $i$ such that $f(\sigma) \in \Class_i$.
\item Then output $i$.
\end{itemize}
\fi
%raf10: andrew, never say trivial!!! i don't think this is even
%correct...you need to be a lot more careful when writing proofs.
%$\C$ trivially satisfies errorless fair treatment; conditioned on
%belonging to a class $f(\sigma) \in \Class_i$, $\sigma$ will always
%receive the same outcome $i$.
\paragraph{Proving that $\C$ satisfies fair treatment:}
Consider some $i \in [m]$ and some class $c \in \Class_i$. We aim to
show that for any two groups $X,Y \in \G_{\P}$ and any outcome $o$,
we have that 
$$\Pr[ \C(\O(\bsigma_X)) = o \mid f(\bsigma_X) = c] = \Pr[\C(\O(\bsigma_Y)) = o \mid f(\bsigma_Y) = c]$$ 
%raf10*: this is what was missing.
First note that, by the first condition in the definition of an
$\epsilon$-trivial context (i.e., the ``equal base rate condition''),
it follows that $c$ is in the support of $\{f(\bsigma_X)\}$ if and only if it
is in the support of $\{f(\bsigma_Y)\}$.
Next, consider some $c$ in the support of $\{f(\bsigma_X)\}$ (and thus
also in the support of $\{f(\bsigma_Y)\}$). Let $i$ be such that $c
\in \Class_i$.
Due to the second condition in the definition of an
$\epsilon$-trivial context (i.e., that the distributions $\braces{\O(\bsigma) \mid f(\bsigma) \in
  \Class_j}$ for $j \in [m]$ have disjoint support), it follows that
for every $\sigma$ such that
$f(\sigma) = c$, $\C(\O(\sigma))$ always outputs $i$. Thus,
$$\Pr[ \C(\O(\bsigma_X)) = i \mid f(\bsigma_X) = c] =
\Pr[\C(\O(\bsigma_Y)) = i \mid f(\bsigma_Y) = c] = 1$$ 
which concludes the proof that $\C$ satisfies $0$-fair treatment.

\paragraph{Proving that $\C$ satisfies rational fairness:}
To show that $\C$ satisfies $2\epsilon$-rational fairness, we note
that, by Claim \ref{clm:FP.RF}, it suffices to show that $\C$
satisfies $\epsilon$-predictive parity. 
%raf10
%Clearly, outcome $\C(O(\sigma)) = i$ is equivalent to $f(\sigma) \in
%\Class_i$.
%So, by the first requirement---i.e., for 
By the first condition of an $\epsilon$-trivial context, we have that for every
$i \in [m]$, $c \in \Class_i$, and $X,Y$ in $\G_{\P}$, 
$$\md(\pr{ f(\bsigma_X) = c \mid
  f(\bsigma_Y) \in \Class_i }, \pr{ f(\bsigma_Y) = c \mid f(\bsigma_Y)
  \in \Class_i }) \leq \epsilon$$
By the disjoint support assumption, we have that $\C(O(\sigma)) = i$
if and only if $f(\sigma) \in \Class_i$, thus we have
%raf10
%$$\md(\pr{ f(\bsigma_X) = c | \C(O(\bsigma_X)) = i }, \prc{
%  f(\bsigma_Y) = c | \C(O(\bsigma_Y)) = i }) \leq \epsilon$$
$$\md(\pr{ f(\bsigma_X) = c \mid \C(O(\bsigma_X)) = i }, \pr{ f(\bsigma_Y) = c \mid \C(O(\bsigma_Y)) = i }) \leq \epsilon$$
so $\C$ satisfies $\epsilon$-predictive parity. % and completing the argument.
\end{proof}

%raf7: why here
%However, as Theorem \ref{thm:impossibility} will show, these contexts are in fact the \emph{only} kinds of contexts which can admit fair classifiers---hence, up to a constant-factor loss, we show that trivial contexts are necessary and sufficient for both of our notions of fairness.

\subsection{Subgroup Perfect Prediction}
To prove the second part of Theorem \ref{thm:impossibility}, we
introduce some additional 
%raf10
%intermediate 
notions.
%raf8
% that we show will prove useful in constructing the partition $\Class_1, \Class_2,
%\ldots, \Class_m$ for a trivial context. 
%
%raf8
%The first of these, \emph{equal base rates}, is loosely analogous to the first condition for a trivial context:
%
%raf10: moved this down
\iffalse
\begin{definition}
%
%raf8
%We will say that a context $\P = (\D, f, g, O)$ satisfies
We say that a context $\P$ has 
$\epsilon$-\define{approximately equal base rates} if 
%raf8
%it is the case that, 
for every $X, Y \in \G_\P$ and every $i \in \Class_\P$, $$\md(\pr{ f(\bsigma_X) = i }, \pr{ f(\bsigma_Y) = i }) \leq \epsilon$$
\end{definition}
\fi

%Essentially, this means that all classes are roughly equally prevalent
%between any pair of groups. 
%raf3: cutting, we now prove this before anyways.
\iffalse
Clearly, if this does hold, it would be
trivial to classify 
%raf3*: why "both"
%both groups in a fair manner by, 
individuals in a
for instance, simply assigning every individual the same outcome or distribution over outcomes.
\fi
%raf8: rewriting
%Next, we present a notion, \emph{subgroup perfect prediction}, which is loosely analogous to the second condition for a trivial context:
\begin{definition}
\iffalse
We will say that a classifier $\C$ satisfies \textbf{subgroup perfect
  prediction} with respect to context $\P = (\D, f, g, O)$ if it is
the case that there exist \emph{proper} subsets $\omega \subset
\Out_\P^\C, \psi \subset \Class_\P$ such that: $$\pr{\C(O(\bsigma))
  \in \omega | f(\bsigma) \in \psi} = 1$$ and $$\pr{\C(O(\bsigma))
  \not\in \omega | f(\bsigma) \not\in \psi} = 1$$
%We will say that a classifier $\C$ satisfies \textbf{subgroup perfect
\fi
We say that a classifier $\C$ satisfies \textbf{subgroup perfect
  prediction} with respect to context $\P$ if there exists a
\emph{proper} subset $\psi \subset \Class_\P$ such that
the distributions 
$$\{\C(O(\bsigma)) \mid f(\bsigma) \in \psi\} \quad \text{and} \quad
\{\C(O(\bsigma)) \mid f(\bsigma) \not\in \psi\}$$ 
have disjoint support.
\end{definition}

%Intuitively, this definition implies that membership in the subgroup
%$\psi$ of classes can be perfectly determined by observing membership
%in the subgroup $\omega$ of classifications. 
%raf8: we didn't define this so seem unnecessary
%Notably, this is also equivalent to \emph{perfect prediction} for binary classifiers $\C$ where $\Class_\P = \Bit$ (as any binary classifier satisfying it is either perfectly predictive or ``perfectly mispredictive", and in the latter case we may without loss of generality invert the output).

%raf8: we don't need this, it is trivial..
\iffalse
We can define an analogous notion for the context itself:

\begin{definition}
We say that a context $\P$ satisfies \textbf{subgroup
  distinguishability} if there exists a proper subset $\psi \subset
\Class_\P$ such that the distributions $\braces{O(\bsigma) |
  f(\bsigma) \in \psi}$ and $\braces{O(\bsigma) | f(\bsigma) \not\in
  \psi}$ have disjoint support.
\end{definition}

%raf7
The following observation will be useful.
%The following claim is trivial but nonetheless useful:

\begin{claim}
A context $\P$ satisfies subgroup distinguishability if and only if there exists a classifier $\C$ satisfying subgroup perfect prediction with respect to $\P$.
\label{claim:dist}
\end{claim}

\begin{proof}
The ``if" direction is immediate; for the ``only if" direction, consider the classifier $\C$ that outputs 1 if $O(\sigma)$ is such that $f(\sigma) \in \psi$ and 0 otherwise; then let $\omega = \braces{1}$.
\end{proof}
\fi

%raf8: putting back
%raf7:
To characterize classifiers satisfying subgroup perfect prediction, a notion
of ``ambiguity between classes'' will be useful. 
%raf7*: i don't like this discussion of clicks. why so long. i will
\iffalse
Consider, given some classifier $\C$ and context $\P = (\D, f, g, O)$, a graph where nodes represent classes in $\Class_\P$ and, for each $k \in \Out_\P^\C$, the nodes $\lbrace i : \pr{f(\bsigma) = i \wedge \C(O(\bsigma)) = k} > 0 \rbrace$ form a clique in the graph. (That is, connections between two nodes exist if and only if they can possibly be classified the same way according to $\C$.) If $\C$ satisfies subgroup perfect prediction, then there exist two or more separate components of the graph, where one such component represents $\psi$. Then we can predict perfectly whether a sampled individual belongs to $\psi$ or not by examining its classification according to $\omega$ and observing whether the corresponding clique lies within the respective component or not. In particular, we define an ``edge" in this graph, as well as the existence of a ``path" between two nodes, as follows:
\fi
\begin{definition}
Given a classifier $\C$ and context $\P$,
%raf8
% = (\D, f, g, O)$, 
we say that
classes $i,j \in \Class_\P$ are \define{ambiguous} (with respect to $\C$ and
$\P$) if there exists $o \in \Out_\P^\C$ such that $\pr{f(\bsigma) = i
  \wedge \C(O(\bsigma)) = o} > 0$ and $\pr{f(\bsigma) = j \wedge
  \C(O(\bsigma)) = o} > 0$. We further say that classes $i,j \in \Class_\P$ are $n$-\define{ambiguous} if there exists a sequence $(i_0 = i, i_1, i_2, \ldots, i_n = j) \in (\Class_\P)^{n+1}$ such that any two consecutive elements $i_k$ and $i_{k+1}$ are ambiguous.
\end{definition}
%That is, the ambiguity between two classes is their shortest connecting path length in the graph we described. Then this graph-based intuition provides the following claim:
We now have the following useful claim which says that, if a classifier
does not satisfy subgroup $\perfpred$, then all classes can be
connected by a ``short'' ambiguous sequence.
\begin{claim}
Consider some classifier $\C$ that does not satisfy subgroup
$\perfpred$ with respect to some context $\P = (\D, f, g, O)$. Then 
%raf7: why have this n?
%there exists some $n \leq k - 1$ such that
for every pair of classes $i,j \in \Class_\P$, we have that $i,j$ are
$m_{i,j}$-ambiguous for some $m_{i,j} \leq |\Class_\P| -1$.
\label{clm:graph}
\end{claim}

\begin{proof}
%raf7:
Given a context $\P$ and a classifier $\C$, consider a graph $G$ with $n =
\Class_{\P}$ vertices, where we draw an edge between two vertices
$i,j$ if $i$ and $j$ are ambiguous. Note that $i,j$ are $m$-ambiguous
if and only if there exists a path of length $m$ connecting them.

We show that the graph must be fully connected if $\C$ does not
satisfy subgroup $\perfpred$; the proof of the claim immediately follows, as
the shortest path between any two nodes in a fully connected graph with
$n$ nodes can never be more than $n-1$.

Assume for the sake of contradiction that $G$ is not fully connected,
yet $\C$ does not satisfy subgroup $\perfpred$.
%raf7 end of new
Then $G$ must have a component $\psi$ disconnected from the remainder
of the graph (which can be concretely found by, say, considering the
set of all vertices reachable from some class $i \in \Class_\P$). Then
we notice that the set 
%$\omega$ 
of outcomes that can be assigned to individuals with classes in $\psi$ must be entirely disjoint from the set of outcomes that can be assigned to individuals with classes outside $\psi$; otherwise, there would by definition exist an edge between a vertex in $\psi$ and a vertex in $\Class_\P \setminus \psi$ in the graph, contradicting our assumption that $\psi$ is disconnected from $\Class_\P \setminus \psi$.
%raf10
%Thus, taking $\psi$ and $\omega$ as we have defined here shows that $\C$ does satisfy subgroup perfect prediction with respect to $\P$, contradicting our assumption and completing the proof.
This contradicts our assumption that $\C$ does not satisfy subgroup perfect prediction.
\iffalse
Notice that, for any $k \in \Out_\P^\C$, all edges corresponding to the classification $\C(O(\sigma)) = k$ must either lie entirely in $\psi$ or entirely outside; this is because the edges for a particular classification must be transitive by the construction of our graph (i.e., they form a clique of all nodes such that an individual of the corresponding class can have $\C(O(\sigma)) = k$), and so, were there such an edge inside $\psi$ and also such an edge outside $\psi$, there would by transitivity be an edge between $\psi$ and its complement, contradicting our assumption of disconnectedness.

Then we can define our subgroup $\omega$ to be the set of all outcomes that produce edges between some clique entirely in $\psi$, as well as any outcomes with no edges but which lie in $\psi$ (i.e., $k \in \Out_\P^\C$ such that, if $\C(O(\sigma)) = k$, then $f(\sigma) = i$ for some single $i \in \psi$). 

Then, given any $\sigma$ such that $f(\sigma) = i \in \psi$ and $\C(O(\sigma)) = k$, it must either be such that $k$ is a ``singleton" classification of the second type, or such that there exists some edge generated by $k$ between $i$ and other nodes, which must by assumption lie in $\psi$; this indicates that $k \in \omega$ as desired. Similarly, any $\sigma$ such that $f(\sigma) = i \not\in \psi$ and $\C(O(\sigma)) = k$ indicates that $k$ satisfies neither of the two properties above and thus may not lie in $\omega$. Hence $\psi$ and $\omega$ must satisfy the definition of subgroup $\perfpred$, proving the claim by contradiction.
\fi
\end{proof}

%raf8
The next lemma can be viewed as a weak form of the second part of
Theorem \ref{thm:impossibility}.
%raf10
(In fact, for the case of binary classification contexts, this
lemma on its own directly implies Theorem 1 from the introduction.)
%raf10* can you double check whether this is true?
Looking forward, we will soon strengthen this lemma by repeatedly applying it to prove
the full Theorem \ref{thm:impossibility}. In the sequel, we say that a context $\P$ has 
$\epsilon$-\define{approximately equal base rates} if 
for every $X, Y \in \G_\P$ and every $i \in \Class_\P$, $$\md(\pr{ f(\bsigma_X) = i }, \pr{ f(\bsigma_Y) = i }) \leq \epsilon$$

\begin{lemma}
%If, for some $\epsilon > 0$, a classifier $\C$ satisfies
%$\epsilon$-$\fairtreat$ and $\epsilon$-predictive parity with respect
%to a context $\P=(\D, f, g, O)$ where $|\Class_\P| = k$, then either:
Let $\P$ be a context, let $\C$ be a classifier that satisfies
$\epsilon$-$\fairtreat$ and $\epsilon$-predictive parity with respect
to a context $\P$, and let $k = |\Class_\P|$. Then either:
\begin{itemize}
\item[(1)] $\P$ satisfies $4 (k-1) \epsilon$-approximately equal base rates, or
\item[(2)] $\C$ satisfies subgroup perfect prediction over $\P$.
\end{itemize}
\label{lemma:impossibility}
\end{lemma}
\begin{proof}
%raf8.
Let $\C$ be a classifier that satisfies
$\epsilon$-$\fairtreat$ and $\epsilon$-predictive parity with respect
to $\P$, and let $k = |\Class_\P|$. We will show that either $\P$ satisfies $4 (k-1) \epsilon$-approximately equal base rates, or
$\C$ satisfies subgroup perfect prediction over $\P$. 
Towards proving the lemma, let us introduce some additional notation,
and prove some helpful propositions:
\begin{itemize}
\item Let $\alpha_X^i = \pr{ f(\bsigma_X) = i }$ denote the base rate
  of the class $i$ w.r.t. the group $X$.
\item Let $f_i$ denote the event that $f(\bsigma) = i$ and let 
$\C_o$
  denote the event that $\C(O(\bsigma)) = o$.
\item For any $X \in \G_\P$, let $f_i^X$ denote the event
  $f(\bsigma_X) = i$ and let $\C_o^X$ denote the event that $\C(O(\bsigma_X)) = o$.
\end{itemize}
The following proposition is a generalization of the identity observed
by Chouldechova in \cite{chouldechova}.
\begin{proposition}
%raf10: k is overloaded, changing k->o when using k as the outcome (k
%is |classes|, not marking it.
Let $i,j \in \Class_\P, o \in \Out_\P^\C$, and $i \neq j$. 
%raf10*: i think this is wrong
%Then, assuming that $\pr{f_i^X } > 0$ and $\pr{f_j^X \wedge \C_o^X } >
%0$:
Then, if $\pr{f_i^X \wedge \C_o^X} > 0$ and $\pr{f_j^X
  \wedge \C_o^X } > 0$, we have:
$$\frac{\pr{\C_o^X \mid f_i^X }}{\pr{\C_o^X \mid f_j^X }} = \frac{\pr{f_j^X }}{\pr{f_i^X }} \frac{\pr{f_i^X \mid \C_o^X }}{\pr{f_j^X \mid \C_o^X }}$$
for any $X \in \G_\P$.
\label{lemma:cond}
\end{proposition}

\begin{proof}
First observe that, if $\pr{f_j^X \wedge \C_o^X } > 0$, then it follows by conditional probability that $\pr{f_j^X \mid \C_o^X } > 0$, $\pr{\C_o^X \mid f_j^X } > 0$, and also $\pr{\C_o^X } > 0$. Then the conclusion follows immediately:
$$\frac{\pr{\C_o^X \mid f_i^X}}{\pr{\C_o^X \mid f_j^X }} = \frac{\pr{\C_o^X \wedge f_i^X } / \pr{f_i^X }}{\pr{\C_o^X \wedge f_j^X } / \pr{f_j^X }}$$ $$= \frac{\pr{f_j^X }}{\pr{f_i^X }} \frac{\pr{f_i^X \mid \C_o^X } \pr{\C_o^X }}{\pr{f_j^X \mid \C_o^X } \pr{\C_o^X }} = \frac{\pr{f_j^X }}{\pr{f_i^X }} \frac{\pr{f_i^X \mid \C_o^X }}{\pr{f_j^X \mid \C_o^X }}$$
\end{proof}

We now use the above proposition to get a relationship between the base rate
of any two classes that are ambiguous.
\begin{proposition}
For any two groups $X, Y \in \G_\P$, and any two classes $i,j \in
\Class_{\P}$ that are ambiguous w.r.t. $\C$, we
have:
$$\md \left( \frac{\alpha_i^X}{\alpha_i^Y}, \frac{\alpha_j^X}{\alpha_j^Y} \right) \leq 4 \epsilon$$
\label{lemma:base}
\end{proposition}

\begin{proof}
%raf10
Consider any two $X, Y \in \G_\P$, and any two classes $i,j \in
\Class_{\P}$ that are ambiguous w.r.t. $\C$. 
By ambiguity, there exists some
% $k \in \Out_\P^\C$ such that $\pr{f_i \wedge \C_{o}} > 0$ and
% $\pr{f_j \wedge \C_{o}} > 0$, implying $\pr{f_i^X \wedge \C_o^X } >
% 0$ and $\pr{f_j^X \wedge \C_o^X } > 0$ for every $X \in \G_\P$ (as it
% must hold for one such group, fair treatment is violated unless it
% holds for all). 
$o \in \Out_\P^\C$ such that 
$$\pr{f_i \wedge \C_{o}} > 0 \quad \quad \text{and} \quad \quad 
 \pr{f_j
  \wedge \C_{o}} > 0.$$
%raf10
There thus must exist groups $g_1, g_2$ such that 
$$\pr{f_i^{g_1} \wedge \C_o^{g_1} } > 0 \quad \quad \text{and} \quad \quad \pr{f_j^{g_2} \wedge
  \C_o^{g_2} } > 0.$$
By fair treatment between the pairs $(g_1, X)$, $(g_2, X)$, $(g_1, Y)$, and $(g_2, Y)$, it follows that
%By fair treatment, it then follows that 
$$\pr{f_i^X \wedge \C_o^X } > 0, \quad \pr{f_j^X \wedge \C_o^X } > 0, \quad
\pr{f_i^Y \wedge \C_o^Y } > 0, \quad \pr{f_j^Y \wedge \C_o^Y } > 0.$$
% for every $X \in \G_\P$ (as it must hold for one such group, fair
% treatment is violated unless it holds for all). 
%raf10
We can thus apply Proposition \ref{lemma:cond} to conclude: 
$$\md(\pr{\C_o^X \mid f_i^X } \alpha_i^X \pr{f_j^X \mid \C_o^X }, \pr{\C_o^X \mid f_j^X } \alpha_j^X \pr{f_i^X \mid \C_o^X}) = 0$$
%raf10
By fair treatment, 
%By the definition of fair treatment:
$\md(\pr{\C_o^X \mid f_i^X }, \pr{\C_o^Y \mid f_i^Y }) \leq \epsilon$ and $\md(\pr{\C_o^X \mid f_j^X }, \pr{\C_o^Y \mid f_j^Y }) \leq \epsilon$, thus\footnote{using the fact that $\md(ab, c) = x$ and $\md(b, d) = y$ implies $\md(ad, c) \leq x + y$} $$\md(\pr{\C_o^Y \mid f_i^Y } \alpha_i^X \pr{f_j^X \mid \C_o^X }, \pr{\C_o^Y \mid f_j^Y } \alpha_j^X \pr{f_i^X \mid \C_o^X }) \leq 2 \epsilon$$ 
%raf10
By  predictive parity, 
%By the definition of predictive parity:
$\md(\pr{f_i^X \mid \C_o^X }, \pr{f_i^Y \mid \C_o^Y }) \leq \epsilon$ and $\md(\pr{f_j^X \mid \C_o^X }, \pr{f_j^Y \mid \C_o^Y }) \leq \epsilon$, thus $$\md(\pr{\C_o^Y \mid f_i^Y } \alpha_i^X \pr{f_j^Y \mid \C_o^Y }, \pr{\C_o^Y \mid f_j^Y } \alpha_j^X \pr{f_i^Y \mid \C_o^Y }) \leq 4 \epsilon$$ 
But, by Proposition \ref{lemma:cond} applied to $Y$ (since $\pr{f_i^Y \wedge \C_o^Y } >
0$ and $\pr{f_j^Y \wedge \C_o^Y } > 0$), it also follows that: 
$$\md(\pr{\C_o^Y \mid f_i^Y } \alpha_i^Y \pr{f_j^Y \mid \C_o^Y }, \pr{\C_o^Y \mid f_j^Y } \alpha_j^Y \pr{f_i^Y \mid \C_o^Y }) = 0$$ 
So, dividing the last two expressions\footnote{using the fact that $\md(a/b, c/d) \leq
  \md(a, c) + \md(b, d)$} (which is possible 
%raf10
since $\pr{f_i^X \wedge \C_o^X } > 0$ and $\pr{f_j^X
  \wedge \C_o^X } > 0$)
we conclude
%because, as mentioned
%before, ambiguity provides that 
%raf8*: wrong notation, canyou fix
%raf10 fixed it
%$\prc{f_j \wedge \C_{o} }{ Y} > 0$), 
$$\md \left( \frac{\alpha_i^X}{\alpha_i^Y}, \frac{\alpha_j^X}{\alpha_j^Y} \right) \leq 4 \epsilon$$
\end{proof}

Armed with the above proposition, we turn to proving the lemma.
Assume for
contradiction that $\P$ does not satisfy $4 (k-1)
\epsilon$-approximately equal base rates, and that
$\C$ does not satisfy subgroup perfect prediction over $\P$. Let $n = k-1$; by Claim
\ref{clm:graph}, we have that, for every pair of classes $i,j \in \Class_\P$, $i$ and $j$ are
$m_{i,j}$-ambiguous for some $m_{i,j} \leq n$.
%raf8* why is the statement below true? i am not thinking just parsing..
%raf10*: so this is an annoying bug it seems...
%First, notice that, in order for fair treatment to hold, it must be
%the case that $\alpha_i^X > 0$ for every $i \in \Class_\P$ and $X \in
%\G_\P$.
By our assumption that $\P$ does not satisfy $4n \epsilon$-$\ebr$,
there exists some $i \in \Class_\P$ and some $X, Y \in \G_\P$ such
that
%raf10
$$\mu(\alpha_i^X,\alpha_i^Y) > e^{4n \epsilon}.$$
%raf10*: this is the fix to the above issue
Thus at least one of $\alpha_i^X$ and $\alpha_i^Y$ needs to be non-zero, and then by the
definition of fair treatment, we have that also the second one must be non-zero.
%it must be the
%case that 
Thus, 
either 
$$\frac{\alpha_i^X}{\alpha_i^Y} > e^{4n \epsilon} \quad \quad \text{or} \quad \quad 
 \frac{\alpha_i^X}{\alpha_i^Y} < e^{-4n\epsilon}.$$
%raf10:
We may assume without loss of generality that the former condition holds
(as we may otherwise switch $X$ and $Y$).
% assume the former, as the proof proceeds symmetrically for the latter.

By our ambiguity assumptions, for any $j \in \Class_\P$, there is some
%raf13:
%$m = m_{i,j} \leq k$ and an ambiguous chain $(i_0 = i, i_1, i_2,
$m = m_{i,j} \leq n$ and an ambiguous chain $(i_0 = i, i_1, i_2,
\ldots, i_m = j) \in (\Class_\P)^{m+1}$ so that any two consecutive
elements $i_s$ and $i_{s+1}$ are ambiguous; in particular this means
that Proposition \ref{lemma:base} applies to any consecutive elements
in the sequence, and thus for every $s \in [m-1]$
$$\md \left( \frac{\alpha_{i_s}^X}{\alpha_{i_s}^Y},
  \frac{\alpha_{i_{s+1}}^X}{\alpha_{i_{s+1}}^Y} \right) \leq 4
\epsilon$$ 
Hence, iteratively employing Proposition \ref{lemma:base}, we have 
$$\frac{\alpha_{i_1}^X}{\alpha_{i_1}^Y} > e^{4(n-1) \epsilon}, \quad
 \frac{\alpha_{i_2}^X}{\alpha_{i_2}^Y} > e^{4(n-2) \epsilon}, \quad \ldots,
 \quad \frac{\alpha_{i_m}^X}{\alpha_{i_m}^Y} = \frac{\alpha_{j}^X}{\alpha_{j}^Y} > e^{4(n-m) \epsilon} \geq 1$$
%raf10*: what does this "restricted O" mean"
%removing
%which, as it applies to any $j$ in our (restricted) $\advO$, indicates
%either that (in this case) $\frac{\alpha_i^X}{\alpha_i^Y} > 1$ for all
%$i \in \Class_\P$ or (in the alternate case where
%$\frac{\alpha_i^X}{\alpha_i^Y} < e^{-4n \epsilon}$) that
%$\frac{\alpha_i^X}{\alpha_i^Y} < 1$ for all $i \in \Class_\P$. 
%raf10:
Thus, for every $j \in \Class_{\P}$, we have that 
$$\alpha_j^X> \alpha_j^Y$$
which is a contradiction since
%But
%this is a contradiction, since 
%we must by the definition of $\Class_\P$ and base rates have that 
$$\sum_{i \in \Class_\P} \alpha_i^X = \sum_{i \in \Class_\P} \alpha_i^Y
= 1.$$
%, which cannot be the case if either of these conditions are true.
\end{proof}

\subsection{Proof of Theorem \ref{thm:impossibility} (2)}
\label{sec:proof1}
%raf8
\iffalse
Finally, towards proving Theorem \ref{thm:impossibility}, we present the main lemma to which we shall devote the final part of this section, and explain why it implies Theorem \ref{thm:impossibility} (2).

\begin{lemma}
If, for some $\epsilon > 0$, a classifier $\C$ satisfies $\epsilon$-$\fairtreat$ and $\epsilon$-predictive parity with respect to a context $\P=(\D, f, g, O)$ where $|\Class_\P| = k$, then either:
\begin{itemize}
\item[(1)] $\P$ satisfies $4 (k-1) \epsilon$-approximately equal base rates, or
\item[(2)] $\C$ satisfies subgroup perfect prediction over $\P$.
\end{itemize}

\label{lemma:impossibility}
\end{lemma}

Given Lemma \ref{lemma:impossibility}, we prove Theorem \ref{thm:impossibility} (2):

\begin{proof}
%raf8: always repeat
%Assume the hypothesis holds for $\C$ and some context $\P$. 
\fi
In this section, we finally prove the second step of Theorem
\ref{thm:impossibility}.
\begin{proposition}[\textit{Theorem \ref{thm:impossibility} (2).}]
Let $\epsilon <3/2$, let $\P$ be a classification context, and let $\C$
be a classifier satisfying $\epsilon$-fair
treatment and $\epsilon / 5$-rational fairness with respect to
$\P$. Then $\P$ is $4(k-1)\epsilon$-trivial, where $k = |\Class_\P|$.
\label{prop:partTwo}
\end{proposition}
\begin{proof}
%raf8*: can you go over this proof. in particular, since i removed the claim
%about subgroup indist, you need to change so you use subgroup perfect
%prediction instead, but it should be trivial given the new def of it.
Consider some classification context $\P$; let
$\epsilon < 3/2$ be a constant and let $k = |\Class_\P|$. Assume
the existence of a classifier $\C$ satisfying $\epsilon$-fair
treatment and $\epsilon / 5$-rational fairness with respect to
$\P$. We aim to show that $\P$ is $4(k-1)\epsilon$-trivial.

%raf8
%Notice that $\epsilon$-predictive parity is implied directly by
%Corollary \ref{cor:RF.FP}, and so we may apply Lemma
%\ref{lemma:impossibility} to $\C$ and $\P$. We shall prove that $\P$
%is $4 (k-1) \epsilon$-trivial.
First, note that by Corollary \ref{cor:RF.FP}, we have that $\C$ also
satisfies $\epsilon$-predictive parity. To show that $\P$ is $4 (k-1)
\epsilon$-trivial, we shall repeatedly apply the following proposition.

\begin{proposition}
Let $\P = (\D, f, g, O)$ be a context (where $|\Class_\P| = k$) for which there exists a classifier $\C$ satisfying $\epsilon$-$\fairtreat$ and $\epsilon$-predictive parity with respect to $\P$.

Let $\Class_1, \ldots, \Class_m$ be a partitioning of $\Class_\P$ into
%raf12
%disjoint proper 
subsets such that, for any $i,j \in [m]$ with $i \neq
j$, the distributions 
$$\braces{\O(\bsigma) \mid f(\bsigma) \in \Class_i} \quad \text{and}
\quad \braces{\O(\bsigma) \mid f(\bsigma) \in \Class_j}$$ 
%raf8 distributions cannot be disjoint, they have disjoint support
%are disjoint. 
have disjoint support.
Then, either of the following conditions must hold.
%raf8: should not be here, i think
%Then, for each $i \in [m]$, either:
\begin{itemize}
\item For any two groups
  $X,Y$ in $\G_{\P}$, 
%raf8
and any $i \in [m]$,
$$\md(\pr{ f(\bsigma_X) = c \mid
    f(\bsigma_X) \in \Class_i }, \pr{ f(\bsigma_Y) = c \mid f(\bsigma_Y)
    \in \Class_i }) \leq 4 (k-1) \epsilon$$
(i.e., $X$ and $Y$ have approximately equal base rates conditioned on each
subset of classes $\Class_i$).
%andrew8: incorrect, for each i in m is quantified above
%raf7: missing quantiication over i
%raf8: hi andrew, i see what you had in mind, but i think it is easier
%to just state this version, which should be easier to prove and
%easier to apply.
%\item There exists a partition of disjoint proper subsets
%\item There exists a some partition of $\Class_i$ into disjoint
\item There exists some $i$ and some partition of $\Class_i$ into 
%disjoint
  \emph{proper} subsets $\Class_i^0$, $\Class_i^1$ such that the
  distributions $\braces{\O(\bsigma) \mid f(\bsigma) \in \Class_i^0}$ and
  $\braces{\O(\bsigma) \mid f(\bsigma) \in \Class_i^1}$ 
%raf8
%are disjoint.
have disjoint support.
\end{itemize}
\label{lemma:induction}
\end{proposition}

\begin{proof}
%raf8*: can you redo the proof given the new statement of the proposition
%raf10
Consider $\P,\C, \Class_1, \ldots, \Class_m$ satisfying the premise of
the proposition. 
Let $\D_i$ be the distribution over $\sigma$ obtained by conditioning
$\D$ on the event that $f(\sigma) \in \Class_i$.
We claim that $\C$ also satisfies $\epsilon$-$\fairtreat$ and
$\epsilon$-predictive parity with respect to each context $\P_i =
(\D_i, f, g, O)$, as the conditional distributions over which they are defined are unchanged if we restrict to $f(\bsigma) \in \Class_i$.
%raf10
%this is immediate because, 
%raf10: "distributions over distributions"?? i understand what you
%want to say...
%for any $c \in \Class_i$, the distributions $\braces{\bsigma \mid f(\bsigma) = c}$ are trivially identical over $\braces{\sigma \leftarrow \D}$ and $\braces{\sigma \leftarrow \D | f(\sigma) \in \Class_i}$, meaning that, for every $c$ represented in $\P_i$, the conditional distributions over $c$ which define fair treatment for $\C$ are perfectly identical to those in $\P$ (and so $\C$ must satisfy $\epsilon$-fair treatment over $\P_i$ if it does over $\P$).
For fair treatment, this is obvious as for any $c \in \Class_i$,
conditioning on $f(\bsigma) = c$ is equivalent to conditioning on
$f(\bsigma) = c \wedge f(\bsigma) \in \Class_i$ (as these events are the same).

For predictive parity, notice that because the distributions
$\braces{\O(\bsigma) \mid f(\bsigma) \in \Class_i}$ are mutually
disjoint, there is also a partition $\Omega_1, \ldots, \Omega_m$ of
the outcome space $\Out_\P^\C$ such that $\C(O(\sigma)) \in \Omega_i$ if and only if $f(\sigma) \in \Class_i$.
Thus conditioning on $f(\bsigma) \in \Class_i$ is equivalent to
conditioning on $\C(\O(\bsigma)) \in \Omega_i$. 
We conclude that
conditioning on $\C(\O(\bsigma)) = o \wedge f(\bsigma) \in \Class_i$
(whenever this event happens with positive probability) is equivalent
to conditioning on $\C(\O(\bsigma)) = o \wedge \C(\O(\bsigma)) \in
\Omega_i$
which in turn is equivalent to conditioning on just $\C(\O(\bsigma)) = o$.
\iffalse
f(\bsigma) \in Class_i$
\in \Omega_i$ , and thus
$\C(\O(\sigma)) \in \Omega_i$
if and only if $f(\sigma) \in \Class_i$. 
So, for $o \in \Omega_i$, distributions $\braces{\sigma \mid \C(O(\sigma)) = o}$ are trivially identical over $\braces{\sigma \leftarrow \D}$ and $\braces{\sigma \leftarrow \D \mid f(\sigma) \in \Class_i} = \braces{\sigma \leftarrow \D \mid \C(O(\sigma)) \in \Omega_i}$, meaning that $\epsilon$-predictive parity holds for $\P_i$ whenever it holds for $\P$.
\fi

Hence, for $\C$ and each context $\P_i$, we can apply Lemma
\ref{lemma:impossibility}, showing that either $\P_i$ has $4 (k-1)
\epsilon$-approximately equal base rates,
%(which directly implies the
%first condition), 
or $\C$ satisfies subgroup perfect prediction with
respect to $\P_i$. 
%raf10
In case all $\P_i$ satisfy  $4 (k-1)
\epsilon$-approximately equal base rates, we are done (we are
satisfying condition 1 in the proposition).
Otherwise, there must exist some $i$ such that $\C$ satisfies subgroup perfect prediction with
respect to $\P_i$; that is, there some proper subset $\psi$ of
$\Class_i$ such that the distributions $\braces{\C(\O(\bsigma)) \mid
  f(\bsigma) \in \psi}$ and 
$\braces{\C(\O(\bsigma)) \mid f(\bsigma) \not\in \psi}$ have disjoint
support (when $\bsigma$ is defined over $\D_i$), which in turn means that 
the distributions $\braces{\O(\bsigma) 
  f(\bsigma) \in \psi}$ and 
$\braces{\O(\bsigma) \mid f(\bsigma) \not\in \psi}$ also have disjoint support.
Hence we may partition $\Class_i$ into $\Class_i^0 = \Class_i
\setminus \psi$ and $\Class_i^1 = \psi$ to satisfy the second
condition of the proposition (relying on the fact that
$\braces{\O(\bsigma) \mid f(\bsigma) \in \Class_i}$ has disjoint support
from the support of $\braces{\O(\bsigma) \mid
  f(\bsigma) \in \Class_j}$ for every $j \neq i$, and thus so will
$\braces{\O(\bsigma) \mid f(\bsigma) \in \Class^b_i}$ for $b \in \{0,1\}$).
\end{proof}

%raf8*: redo the proof becasue of the new proposition??
%raf10: it seems ok.
Now, noticing that we may partition $\Class_\P$ into at most
$|\Class_\P| = k$ distinct subsets,
%raf10: note sure what this is relevant (which is surprising, but this
%should have already been inside the proof of the proposition)
% (in particular, if $|\Class_i| =
%1$, the first condition in Proposition \ref{lemma:induction} trivially
%holds as all base rates are 1), 
we can apply Proposition
\ref{lemma:induction} repeatedly at most $k-1$ times (starting with
$\Class_1 = \Class_\P$, every time increasing the number of
partitions by one (by replacing $\Class_i$ with $\Class_i^0$ and
$\Class_i^1$). Thus, when we can no longer further partition some
subset $\Class_i$, the first condition from the proposition must hold,
and thus we have ``$4 (k-1) \epsilon$-approximately equal base rates conditioned on $\Class_i$'' for
every $i$. We conclude that $\P$ is $4 (k-1) \epsilon$-trivial, which completes the proof of Theorem \ref{thm:impossibility} (2).

%\end{proof}
%raf8
%Together, Claim \ref{clm:partOne} and Lemma \ref{lemma:impossibility} complete the proof of Theorem \ref{thm:impossibility}. So it suffices to prove Lemma \ref{lemma:impossibility}; we shall do this now.
\end{proof}

%raf8: added, can you fix the refs.
\subsection{Concluding the Proof of Theorem \ref{thm:impossibility}}
Theorem \ref{thm:impossibility} follows as a direct consequence of
%raf10: can you fill in the missing refs.
Proposition \ref{clm:partOne} and Proposition \ref{prop:partTwo}. This concludes the proof of the
main theorem.

\label{sec:proof1}

%raf8*: this section is a total mess. e.g. see the proposition below,
%ut uses things that are not even defined. ambiguituity...pulling
%things back to where they were and commenting out
\section*{Acknowledgments}
We are extremely grateful to Jon Kleinberg, Silvio Micali and Ilan Komargodski for
delightful conversations and insightful suggestions.

\bibliographystyle{plain}
\bibliography{balance}

%raf7: dude, never hack latex!! use \appendix :-)
%\def\thesection{\Alph{section}}
%\setcounter{section}{0}
\appendix

\section{Fair Treatment and Post-Processing}
\label{threshold.ssec}
We here remark that the notion of fair treatment is closed under
``post-processing''. If a classifier $\C$ satisfies
$\epsilon$-fair treatment with respect to a context $\P = (\D, f, g,
O)$, then for any (possibly probabilistic) function 
$\M$, $\C'(\cdot) = \M(\C(\cdot))$ will
also satisfy $\epsilon$-fair treatment with respect to $\P$. 
\begin{theorem}
Let $\C_1$ be a classifier satisfying $\epsilon$-$\fairtreat$ with
respect to context $\P = (\D, f, g, O)$. Let $\C_2$ be any classifier
%raf9:
%over the context $\P' = (\D, f, g, \C_1)$
over the context $\P' = (\D, f, g, \C_1(O(\cdot)))$
%raf9: you cannot have footnotes in a theorem that explain it if you
%don't know how to state it formally.
%\footnote{That is, the only feature of an individual which is
%observable by $\C_2$ is the output of $\C_1$.}. 
Then $\C_2$ satisfies $\epsilon$-$\fairtreat$ with respect to $\P'$.
\label{clm:bracket}
\end{theorem}

\begin{proof}
%raf9: i don't know what this means
%Assume without loss of generality that $\C_2$, given input $o_1 \in \Out_\P^{\C_1}$, outputs the probability distribution $M(o_1) \in \Delta(\Out_\P^{\C_2})$.
%We first repeat the analysis above. 
%raf9*: sorry andrew, but this proof is not understandable and uses old
%notation. i am rewriting it below, please check it carefully and try
%to apply a similar style in the remaining proofs.
%raf9: always start by explaining what you want to show.
Let $\C_1$ be a classifier satisfying $\epsilon$-$\fairtreat$ w.r.t. some
context $\P$. Consider some groups $X,
Y \in \G_\P$, some class $c \in
\Class_{\P}$, and some outcome $o \in \Out_\P^{\C'}$; we need to show
that 
  $$\mu(\Pr [\C_2(\C_1(\O(\bsigma_X))) = o \mid f(\bsigma_X) = c], \Pr
  [\C_2(\C_1\O(\bsigma_Y))) = o \mid f(\bsigma_Y) = c]) \leq \epsilon$$
%raf9: then show is formally.
Towards doing this, note that 
\begin{eqnarray*}
\Pr [\C_2(\C_1(\O(\bsigma_X))) = o \mid f(\bsigma_X) = c] \\
= \sum_{o_1
  \in \Out_\P^{\C_1}} \Pr [\C_2(o_1) = o \mid f(\bsigma_X) = c, C_1(o) =
  o_1] \Pr [\C_1(\O(\bsigma_X)) = o \mid f(\bsigma_X) = c]
\\
= \sum_{o_1 \in \Out_\P^{\C_1}} \Pr [\C_2(o_1) = o \mid C_1(o) =
  o_1] \Pr [\C_1(\O(\bsigma_X)) = o \mid f(\bsigma_X) = c]
\end{eqnarray*}
By the same argument applied to $Y$, we also have that:

\begin{eqnarray*}
\Pr [\C_2(\C_1(\O(\bsigma_Y))) = o \mid f(\bsigma_Y) = c]\\
= \sum_{o_1
  \in \Out_\P^{\C_1}} \Pr [\C_2(o_1) = o \mid C_1(o) =
  o_1] \Pr [\C_1(\O(\bsigma_Y)) = o \mid f(\bsigma_Y) = c]
\end{eqnarray*}

These two probabilities are $\epsilon$-close since, by fair treatment, $\Pr [\C_1(\O(\bsigma_X)) = o \mid f(\bsigma_X) = c]$
and $\Pr [\C_1(\O(\bsigma_Y)) = o \mid f(\bsigma_Y) = c]$ are
$\epsilon$-close, and furthermore by the fact that multiplicative distance is
%raf9*: andrew, can you check this identity? won't you need to use it
%raf10*: still please check and respond here.
%also in the pred parity implies rational fairness? now that we know
%what identities for \mu that we use, we may want to put them up front
%(if we reuse them)
preserved under linear operations\footnote{That is, if $\mu(a,b) \leq
  \epsilon$ and $\mu(a',b') \leq \epsilon$ then $\mu(\alpha a + \beta
  a', \alpha b + \beta
  b') \leq \epsilon$.}. This proves the theorem.
\end{proof}

\paragraph{Expectation-based notions are not closed under
  post-processing.} Let us also remark that earlier
``expectation-based'' definitions of
$\fairtreat$ are not preserved under post-processing. 

Assume we have some non-binary finite outcome space of scores $\Omega
\subset [0,1]$, similar to the notion of ``risk scores" proposed in
\cite{kleinberg}, and that we define $\fairtreat$ as requiring that,
conditioned on a particular class, the \emph{expected} outcome for an
individual must be within $e^\epsilon$ multiplicatively between the
two groups (as the notion considered in \cite{kleinberg}). Now consider the following example of a threshold classifier which does not preserve this definition of $\fairtreat$:

\begin{itemize}
\item Assume $\advO = \braces{0, 0.25, 0.5, 0.75, 1}$, and that, for some class $c$, individuals in groups $X$ and $Y$ with class $c$ are classified by some classifier $\C$ according to the following distributions:
\begin{center}
\medskip
\begin{tabular}{|l|c|c|c|c|c|}
\hline
 & 0 & 0.25 & 0.5 & 0.75 & 1\\
\hline
Group $X$ & 50\% & 0\% & 0\% & 50\% & 0\%\\
\hline
Group $Y$ & 50\% & 0\% & 25\% & 0\% & 25\%\\
\hline
\end{tabular}
\end{center}

\item Clearly, the expected score $\C(\sigma)$ of individuals in class $c$ is the same for groups $X$ and $Y$ (0.375 for both); if we assume that all individuals in other classes are assigned a random classification by $\C$, then we observe that $\C$ satisfies errorless $\fairtreat$ by the variant definition above.

\item However, if we create a threshold classifier $\C'$ that outputs 0 if $\C(\sigma) < 0.75$ and 1 if $\C(\sigma) \geq 0.75$, then we notice that, of individuals in class $c$, 50\% in group $X$ receive a score of 1 under $\C'$, while only 25\% in group $Y$ receive a score of 1. Hence, $\C'$ does not preserve the same $\fairtreat$ error as $\C$ (as the expected scores of groups $X$ and $Y$ conditioned on $c$ are 0.5 and 0.25, respectively).
\end{itemize}

Intuitively, the reason for this failure in closure under
post-processing is that an expectation-based notion of $\fairtreat$
does not necessarily account for base differences in the
\emph{distributions} of outcomes (conditioned on a particular
class)---indeed, two distributions may have the same expectation but
be quite different. We lastly note that this issue is not present for
binary classifiers (i.e., classifiers that output just a single bit),
since for such classifiers the expectation-based definition is
equivalent to our definition.

%raf7
%So, a classifier that makes decisions based solely on a score threshold or score brackets of individuals will inherit the same (or better) $\fairtreat$ error from the scores from which it is derived.

\section{On the Use of Multiplicative Distance}
\label{sec:varcomp}
%raf7: i made super quick edits, please go over it carefully
%The above shows that, in order to prove Theorem \ref{clm:bracket}, one
%must use a distribution-based notion of $\fairtreat$. However, also a
%distribution-based definition of $\fairtreat$ where we measure
%distance between the distribution using ``statistical distance''
%(i.e., the the total variation norm) will suffice to show closure
%under post processing. 
%The reason we chose to use max-divergence (i.e.
%a multiplicative notion of distance) as opposed to an additive notion of distance (such as
%statistical distance, i.e., the total variation norm) is
%that many classifications that are inherently discriminatory have very
%small additive $\fairtreat$ errors when small probabilities are
%inherent in the distribution or the classifier. 

We here motivate our use of a multiplicative notion of distance (i.e.,
max-divergence)  as opposed to an additive notion of distance (such as
statistical distance, i.e., the total variation norm).
Consider a classifier used to determine whether to search people for
weapons. Assume such a classifier determined to search 1\% of
minorities at random, but \emph{only} the minorities (and no
others). Such a classifier would still have a $\fairtreat$ error of 0.01 if we used the total variation norm, while the max-divergence would in fact be infinite (and indeed, such a classification would be blatantly discriminatory).

Our use of max-divergence between distributions for our definitions is
reflective of the fact that, in cases where we have such small
probabilities, discrimination should be measured multiplicatively,
rather than additively. In addition, when we may have a large number
of possible classes, the use of max-divergence (in particular, the
\emph{maximum} of the log-probability ratios) means that we always
look at the class with the \emph{most} disparity to determine how
discriminatory a classification is, rather than potentially amortizing
this disparity over a large number of classes. 
%So the definitions of fairness that we provide in Section \ref{sec:definitions} are not only useful to our proofs and analysis, but also natural in that they satisfy several desiderata that are not true of other formulations.

\end{document}